\documentclass[twoside]{article}

\usepackage{aistats2021}

\usepackage[utf8]{inputenc} 
\usepackage[T1]{fontenc}    
\usepackage{hyperref}       
\usepackage{url}            
\usepackage{booktabs}       
\usepackage{amsfonts}       
\usepackage{nicefrac}       
\usepackage{microtype}      
\usepackage{graphicx}
\usepackage{subcaption}    
\usepackage{color} 
\usepackage{floatrow}
\usepackage{caption}

\usepackage{amsmath}
\usepackage{amsthm}
\usepackage{amssymb}
\usepackage{thmtools, thm-restate}
\DeclareMathOperator*{\argmin}{arg\,min\,}
\newcommand{\R}{\mathbb{R}}
\newcommand{\N}{\mathbb{N}}

\newcommand{\E}{\mathbb{E}}
\usepackage{graphicx}
\newtheorem{theorem}{Theorem}[section]

\newtheorem{lemma}[theorem]{Lemma}

\newtheorem{remark}{Remark}[section]
\usepackage{xcolor}
\usepackage{stmaryrd}
\usepackage{algorithmic}
\usepackage[ruled,vlined]{algorithm2e}
\newcommand{\bmY}{\mathcal{Y}}
\newcommand{\bmA}{\mathcal{A}}
\newcommand{\bmB}{\mathcal{B}}
\newcommand{\bmX}{\mathcal{X}}
\newcommand{\bmH}{\mathcal{H}}
\newcommand{\bmG}{\mathcal{G}}
\newcommand{\bmF}{\mathcal{F}}
\newcommand{\bmS}{\mathcal{S}}
\newcommand{\bmL}{\mathcal{L}}

\newcommand{\bmU}{\mathcal{U}}

\newcommand{\bmO}{\mathcal{O}}
\newcommand{\bmR}{\mathcal{R}}
\newcommand{\loss}{\Delta}

\newcommand{\kerx}{{k_x}}
\newcommand{\kery}{k}
\newcommand{\bmHy}{\bmH_y}
\newcommand{\bmHK}{\bmH_{\K}}
\newcommand{\bmHx}{\bmH_x}
\newcommand{\K}{K}

\newcommand{\hbpsi}{h^{*}_{\psi}}

\newcommand{\hhatpsi}{\hat{h}_{\psi}}

\usepackage[round]{natbib}

\begin{document}

%

%
\twocolumn[

\aistatstitle{Learning Output Embeddings in Structured Prediction}
\aistatsauthor{Luc Brogat-Motte$^1$ \And   Alessandro Rudi$^2$ \And  C\'eline Brouard$^3$ \And   Juho Rousu$^4$ \And   Florence d'Alch\'e-Buc$^1$}

\aistatsaddress{}]

\begin{abstract}
A powerful and flexible approach to structured prediction consists in embedding the structured objects to be predicted into a feature space of possibly infinite dimension by means of output kernels, and then, solving a regression problem in this output space. A prediction in the original space is computed by solving a pre-image problem. In such an approach, the embedding, linked to the target loss,  is defined prior to the learning phase. In this work, we propose to jointly learn a finite approximation  of  the  output embedding and the regression function into the new feature space.
For that purpose, we leverage a priori information on the outputs and also unexploited unsupervised output data, which are both often available in structured prediction problems. We prove that the resulting structured predictor is a consistent estimator, and derive an excess risk bound. Moreover, the novel structured prediction tool enjoys a significantly smaller computational complexity than former output kernel methods. 
The approach empirically tested on various structured prediction problems reveals to be versatile and able to handle large datasets.
\end{abstract}

\section{INTRODUCTION} \label{sec:introduction}
A large number of real-world applications involves the prediction of a structured output \citep{nowozin11}, whether it be a sparse multiple label vector in recommendation systems \citep{TsoumakasK07}, a ranking over a finite number of objects in user preference prediction \citep{hullermeier2008} or a labeled graph in metabolite identification \citep{nguyen2019}. Embedding-based methods generalizing ridge regression to structured outputs \citep{weston2003kernel,cortes2005,Brouard_icml11,kadri2013,brouard2016input,Ciliberto2016}, besides conditional generative models and margin-based methods \citep{tsochantaridis2004,taskar2004,bakhtin2020}, represent one of the main theoretical and practical frameworks to solve structured prediction problems and also find use in other fields of supervised learning such as zero-shot learning \citep{palatucci2009}. 

In  this work, we focus on the framework of Output Kernel Regression (OKR) \citep{geurts2006,brouard2016input} where the structured loss to be minimized depends on a kernel, referred as the {\it output kernel}.  OKR relies on a simple idea: structured outputs are embedded into a Hilbert space (the canonical feature  space associated to the output kernel), enabling to substitute to the initial structured output prediction problem, a less complex problem of vectorial output regression. Once this problem is solved,  a structured prediction function is obtained by decoding the embedded prediction into the original output space,  e.g. solving a pre-image problem. To benefit from an infinite dimensional embedding, the kernel trick is leveraged in the output space, opening the approach to a large variety of structured outputs. 


A generalization of the OKR approaches under the name of Implicit Loss Embedding has  been recently studied from a statistical point of view in \citep{ciliberto2020general}, extending the theoretical guarantees developed in  \citep{Ciliberto2016,nowak2019general} about the Structure Encoding Loss Framework (SELF). In particular, this work proved that the  excess risk of the final structured output predictor depends on the excess risk of the surrogate regression estimator. This motivates the approach of this paper, controlling the error of the regression estimator by adapting the embedding to the observed data.

In this work, we propose to jointly learn a finite dimensional embedding that approximates the given infinite dimensional embedding and regress the new embedded output variable instead of the original embedding.
Our contributions are four-fold:
\begin{itemize}
    \item We introduce,  Output Embedding  Learning  (OEL), a novel approach to Structured Prediction that jointly learns a finite dimensional embedding of the outputs and the regression of the embedded output given the input, leveraging the prior information about the structure and unlabeled output data.
    \item We devise an OEL approach focusing on kernel ridge regression and a projection-based embedding that exploits the closed-form of the regression problem. We  provide an efficient learning algorithm based on randomized SVD and Nystr\"om approximation of kernel ridge regression.
    \item For this novel estimator, we prove its consistency and derive excess risk bounds.
    \item We provide a comprehensive experimental study on various Structured Prediction problems with a comparison with dedicated methods, showing the versatility of our approach. We particularly highlight the advantages of OEL when unlabeled output data are available while the labeled dataset is limited. However, even when only using labeled data, OEL is shown to reach similar state-of-the-art results with a drastically reduced decoding time compared to OKR.
\end{itemize}


\section{OUTPUT EMBEDDING LEARNING}\label{sec:pb-setting}

{\bf Notations:~}$\bmX$ denotes the input space and $\bmY$ is the set of structured objects of finite cardinality $| \bmY |=D$. Given two spaces $\bmA$, $\bmB$, $\bmF(\bmA, \bmB)$ denotes the set of functions from $\bmA$ to $\bmB$. Given two Hilbert spaces $\bmH_1$ and $\bmH_2$, $\bmB(\bmH_1,\bmH_2)$ is the space of bounded linear operators from $\bmH_1$ to $\bmH_2$. $Id_{\bmH_1}$ is the identity operator over $\bmH_1$. The adjoint of an operator $A$ is noted $A^*$. 
\subsection{Introducing OEL}
Structured Prediction is generally associated to a loss $\loss: \bmY \times  \bmY \to \mathbb{R}$ that takes into account the inherent structure of objects in $\bmY$. In this work, we consider a structure-dependent loss by relying on an embedding $\psi: \bmY \to \bmHy$ that maps the structured objects into a Hilbert space $\bmHy$  and  the squared loss defined over pairs of elements of $\bmHy$: $\loss(y,y')=\| \psi(y) - \psi(y')\|^2_{\bmHy}$.

A principled and general way to define the embedding $\psi$ consists in choosing  $\psi(y) = \kery(\cdot,y)$, the canonical feature map of a positive definite symmetric kernel $\kery$ defined over $\bmY$, referred here as the {\it output kernel}. The space $\bmHy$ is then the Reproducing Kernel Hilbert Space associated to kernel $\kery$. This choice enables to solve various structured prediction problems within the same setting, drawing on the rich kernel literature devoted to structured objects \citep{gartner08}.

Given an unknown joint probability distribution $P(X,Y)$ defined on $ \mathcal{X} \times \mathcal{Y}$, the goal of structured prediction is to solve the following learning problem:
\begin{equation}\label{eq:pb-target}
\min\limits_{f \in \bmF(\bmX,\bmY)}~\E_{X,Y \sim P}\left[ \| \psi(Y) - \psi(f(X)) \|^2_{\bmHy} \right],
\end{equation}
with the help of a training i.i.d. sample $\bmS_n =\{(x_i,y_i), i=1, \ldots n\}$ drawn from $P$. 


To overcome the  inherent difficulty of learning $f$ through $\psi$,  Output Kernel Regression address Structured Prediction as a surrogate problem, e.g. regressing the target  variable $\psi(Y)$ given $X$, and then make their prediction in the original space $\bmY$ with a decoding function as follows (see Figure \ref{fig:schema}, left):

\begin{equation}
\label{eq:reg-surrogate}
\hbpsi \in \argmin_{h \in \bmF(\bmX,\bmHy)}~\E_{X,Y \sim P}\left[\| \psi(Y) - h(X) \|^2_{\bmHy} \right],
\end{equation}
This regression step is then followed by a pre-image or {\it decoding} step in order to recover $f^*$:
\begin{equation*}
  f^*(x) = d \circ \hbpsi(x), 
\end{equation*}
where the decoding function $d$ computes $d(z)= \arg \min_{y \in \bmY} \| z - \psi(y) \|^2_{\bmHy}$.

While the above is a powerful approach for structured prediction, relying on a fixed output embedding given by $\psi$ may not be optimal in terms of prediction error, and it is hard by a human expert to decide on a good embedding. 

\begin{figure}[t]
     \centering
     \begin{subfigure}[b]{0.4\textwidth}
         \centering
         \includegraphics[width=\textwidth]{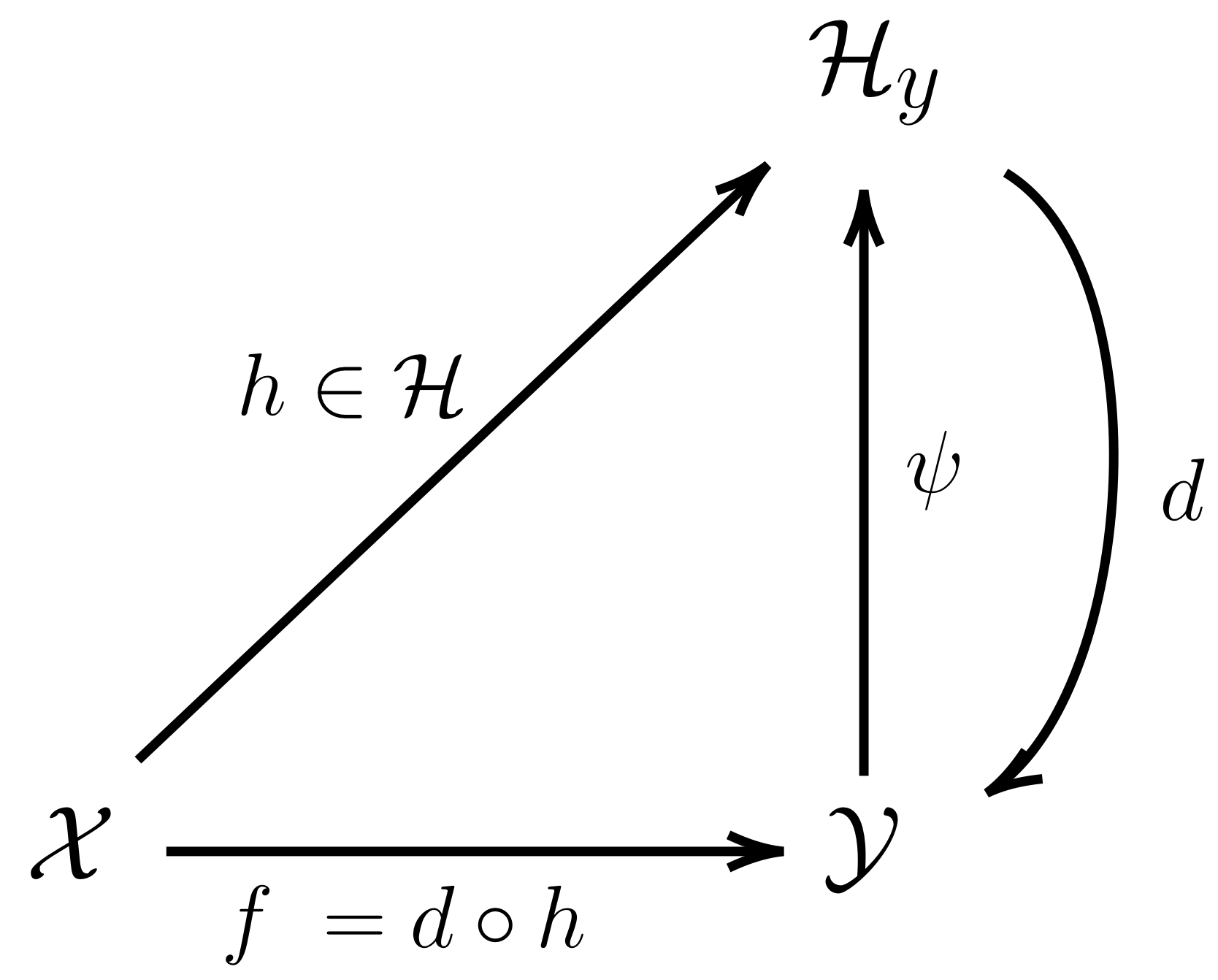}
         \caption{OKR}
         \label{fig:y equals x}
     \end{subfigure}
    \hspace{0.5cm}
     \begin{subfigure}[b]{0.4\textwidth}
         \centering
         \includegraphics[width=\textwidth]{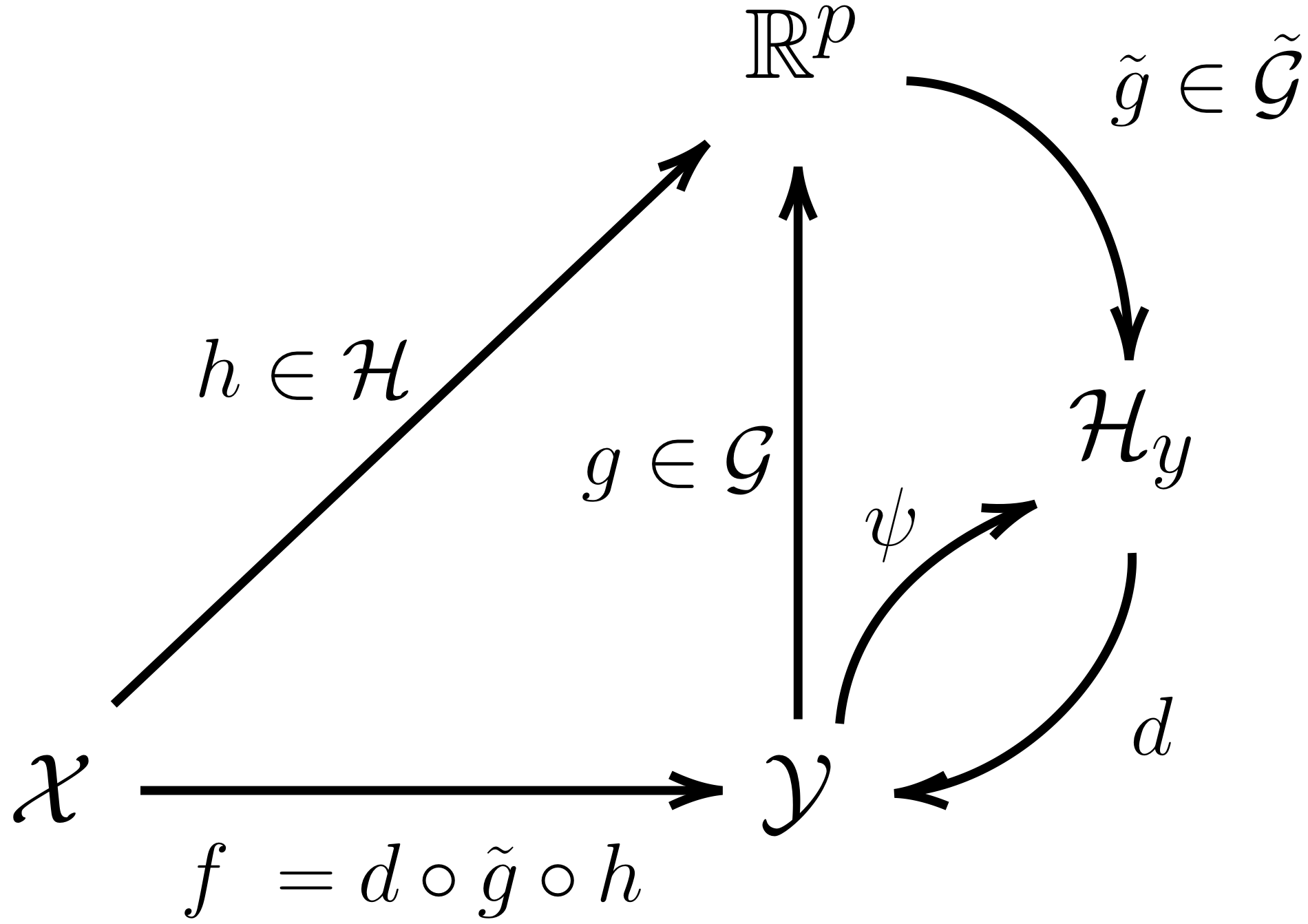}
         \caption{OKR with OEL}
         \label{fig:three sin x}
     \end{subfigure}
     \hfill
     \caption{Schematic illustration of OKR and OKR with OEL}
     \label{fig:schema}
\end{figure}

In this paper, we propose to jointly learn a novel output embedding $g: \bmY \rightarrow \mathbb{R}^p$ as a finite dimensional proxy  of $\psi$ and the corresponding regression model $h: \bmX \to \mathbb{R}^p$.

Our novel approach, called Output Embedding Learning (OEL), thus consists in solving the two problems (Figure \ref{fig:schema}, right).\\
{\bf Learning:} $\gamma \in [0,1]$, minimize w.r.t $h, g, \tilde g$
\begin{align}
\label{eq:oel-gen}
\begin{split}
   &\gamma \underbrace{\E_{X,Y} [\| h(X) - g(Y) \|^2_{\mathbb{R}^p}]}_{\bmR(h,g)}\\ &+ (1- \gamma) \underbrace{\E_Y [\| \tilde g \circ g(Y) - \psi(Y)\|^2_{\bmHy}]}_{\Gamma(g)},
 \end{split}
 \end{align}
{\bf Decoding:}
\begin{equation}
f^*(x) = d \circ {\tilde g}^* \circ h^*_{g^*}(x).
\end{equation}
In the learning objective of Eq. \eqref{eq:oel-gen}, the term $\bmR(h,g)$ expresses a surrogate regression problem from the input space to the learned output embedding space
while $\Gamma(g)$ is a reconstruction error that constrains $g$ to provide a good proxy of $\psi$, and thus, encouraging the novel surrogate loss $(h(x),y) \to \| h(x) - g(y)\|^2$ to be calibrated with the loss $\loss$.

This approach allows learning an output embedding that, intuitively,  is easier to predict from inputs than $\psi$ and also provides control of the complexity of surrogate regression model $h$ by choosing the dimension $p$. 

Concerning the decoding phase, one can be surprised not to use directly $h^*_{g^*}(x)$ by considering a decoding function using the $\ell_2$ norm in $\mathbb{R}^p$. The reason to call for $\tilde g: \R^p \rightarrow \bmH_y$ is driven by theory. This choice allows us to derive an excess risk for the corresponding prediction function, empirically estimated.

To solve the learning problem in practise, a training i.i.d. sample $\bmS_n =\{(x_i,y_i), i=1, \ldots n\}$ is used for estimating $\bmL(h,g)$. For estimating $\Gamma(g)$ we can also benefit from additional i.i.d. samples of the outputs $Y$, denoted $\bmU_m$. Such data is generally easy to obtain for many structured output problems, including the metabolite identification task described in the experiments.

\subsection{Solving OEL with a linear transformation of the embedding}\label{subsec:oel_linear}
 We consider the case where the chosen model for the output embedding is a linear transformation of $ \psi $: $g(y)=G \psi(y)$ where  $G$ is an operator, $G \in \bmG_p = \{ A \in \bmL(\bmHy,\R^p), AA^*=Id_p\}$, with the linear associated decodings: $\tilde g(z) = G^*z~\in \bmHy$. Here $\tilde g \circ g$ can be interpreted as a one-layer linear autoencoder whose inputs and outputs belong to the Hilbert space $\bmHy$ (thus giving overall non-linear embedding $g$) \citep{laforgue2018autoencoding}, and the hidden layer is trained in supervised mode (through $h$), or alternatively as a Kernel PCA model \citep{scholk98} of the outputs, but trained in supervised mode.
 
We denote  $\hbpsi(x)$ the conditional expectation of $\psi(Y)$ given $x$, $\hbpsi(x) = \E_{y}[\psi(y)|x]$. Within this setting, the general problem depicted in Eq. \eqref{eq:oel-gen}  instantiates as follows:
\begin{equation}\label{eq:linear-pb-1}
\begin{split}
  \min\limits_{h \in \bmF(\bmX, \mathbb{R}^p), G \in \bmG_p} &\gamma \E_{X,Y} [\| h(X) - G\psi(Y) \|^2_{\R^p}] \\& \quad + (1- \gamma) \E_Y [\| G^*G \psi(Y) - \psi(Y)\|^2_{\bmHy}]
  \end{split}
\end{equation}

Leveraging the regression $\hbpsi(x)$, and $\| G \psi(y)\|_{\R^p} = \| G^*G \psi(y)\|_{\bmHy}$ ($G$ is orthogonal), this is equivalent to solve the following subspace learning problem\footnote{see details in Section 1 of the supplements}:
\begin{equation}\label{eq:linear-pb-3}
 \begin{split}
\min\limits_{G} & \:\gamma~ \E_{X} [\| G^*G \hbpsi(X) -  \hbpsi(X)\|^2_{\bmHy}] \\ & \quad+ (1- 2\gamma) \E_Y [\| G^*G\psi(Y) - \psi(Y)\|^2_{\bmHy}] ,
 \end{split}
 \end{equation}
 where we restrict $\gamma \leq \frac{1}{2}$ to ensure that the objective is  theoretically grounded. In the following, we use $\gamma_c := \frac{c}{1+c} \leq \frac{1}{2}$, with $c \in [0,1]$. The objective boils down to estimating the linear subspaces of the $\E_{y|x}[\psi(y)]$ and the $\psi(y)$. 

In the empirical context, given an i.i.d. labeled sample $\{(x_i,y_i), i=1, \ldots n\}$, and an  i.i.d unlabeled sample $\{y_j, j=1, \ldots m \}$, we propose to use an empirical estimate $\hhatpsi$ of the unknown conditional expectation $\hbpsi$ and solve the following remaining optimization problem in $G$:
\begin{equation}\label{eq:linear-emp-3}
\begin{split}
  \min\limits_{G \in \bmG_p} \frac{c}{n}
 &\sum_{i=1}^n \| G^*G \hhatpsi(x_i) - \hhatpsi(x_i)\|^2_{\bmHy}  \\& \quad + \frac{(1- c)}{m} \sum_{j=1}^m\| G^*G \psi(y_j) - \psi(y_j)\|^2_{\bmHy}.
\end{split}
\end{equation}

{\bf Learning in vector-valued RKHS:}~
To find an  empirical estimate  of $\hbpsi$, we need a hypothesis space $\bmH \subset \bmF(\bmX, \bmHy)$, whose functions have infinite dimensional outputs. Following the Input Output Kernel Regression (IOKR) approach depicted in \cite{brouard2016input}, we solve a kernel ridge regression problem in $\bmHK$, the RKHS associated to the operator-valued kernel $\K(x,x')  = Id_{\bmHy} \kerx(x,x')$ and we got the following closed-form expression:
\begin{equation}\label{eq:hIOKR}
    \hhatpsi (x) = \sum_{i=1}^n \alpha_i(x) \psi(y_i),\,\textrm{ with } {\boldsymbol \alpha}(x) = (K_x + n \lambda I)^{-1} \kappa_X^x,
\end{equation}
where $\kappa_X^x = [\kerx(x_1,x),\ldots,\kerx(x_n,x)]^T$ and $\lambda > 0$ is the ridge regularization hyperparameter.

{\bf OEL estimator:}~
For a given $(c, \lambda, p)$, denoting $\hat{G}_p$ as the solution of the above problem, the proposed estimator for the solution of the problem stated in Eq. \eqref{eq:linear-pb-1} can be expressed as: $\hat{h}_{\hat{G}_p}(x)= \hat{G}_p\hhatpsi(x)$.
However it is important to stress that we only need to compute the associated structured prediction function:
\begin{align}\label{eq:h-solution}
    \hat f(x) &= \arg \min_{y \in \bmY} \|\hat{G}_p^*\hat{G}_p\hhatpsi(x) - \psi(y)\|^2_{\bmH_y}.
\end{align}


We derive Algorithm \ref{algo:linear} which consists in computing the singular value decomposition of the mixed gram matrix $K$, noticing that the objective \eqref{eq:linear-emp-3} is equivalent to minimize the empirical mean reconstruction error of the $n+m$ vectors of $\bmH_y$: 
$\left(\left(\sqrt{\frac{c}{n}}\hhatpsi(x_i)\right)_{i=1}^n, \left(\sqrt{\frac{(1-c)}{m}}\psi(y_i)\right)_{i=1}^m\right)$.\\

\begin{algorithm}
   \caption{Output Embedding Learning with KRR (Training)}
\begin{algorithmic}
   \STATE {\textbf{Input:}} $K_x, K_y^{s,s} \in \R^{n \times n}$ (supervised data), $K_y^{u,u} \in \R^{m \times m}, K_y^{s,u} \in \R^{n \times m}$ (unsupervised data), $ \lambda \geq 0 $ KRR regularization, $p \in \N^*$ embedding dimension, $c \in [0, 1]$ supervised/unsupervised balance.
   \STATE {\textbf{KRR estimation:}} $ W = (K_x + n\lambda I)^{-1} $ / $K_h = WK_xK_y^{s,s}K_xW$ / $K_{hy} = WK_xK_y^{s,u}$
   \STATE {\textbf{Subspace estimation:}}
   \STATE 1) $K = 
            \begin{bmatrix}
            \frac{c}{n} K_h & \sqrt{\frac{c(1-c)}{nm}} K_{hy} \\
            \sqrt{\frac{c(1-c)}{nm}} K_{hy}^T & \frac{(1- c)}{m} K_y^{u,u}
            \end{bmatrix} \in \R^{(n + m) \times (n + m)}$
   
   \STATE 2) $\beta = 
            \begin{bmatrix}
            \vert & & \vert \\
            \frac{u_1}{\sqrt{\mu_1}}  & \dots  & \frac{u_p}{\sqrt{\mu_p}}   \\
            \vert & & \vert
            \end{bmatrix} \in \R^{(m + n) \times p} \gets SVD(K) = \sum_{l=1}^{n+m} \mu_l u_lu_l^T$
    \STATE 3) $GY = K\beta$ 
    
   \RETURN $W$ KRR coefficients, $\beta$ output embedding coefficients, GY new training embedding

\end{algorithmic}
\label{algo:linear}
\end{algorithm}

\paragraph{Training computational complexity} The complexity in time of the training Algorithm \ref{algo:linear} is the sum of the complexity of a Kernel Ridge Regression (KRR) with $ n$ data and a Singular Value Decomposition  with $ n + m $ data: $\bmO(n^3) +  \bmO((n + m)^3)$. However, this complexity can be a lot improved  as for both KRR and SVD there exists a rich literature of approximation methods \citep{rudi2017falkon, halko2011finding}. For instance, using Nyström KRR approximation of rank $ q $ and randomized SVD approximation of rank $p$, then, the time complexity becomes: $\bmO(n q^2) + \bmO((n + m)p^2)$.
\begin{table}[ht!]
    \centering
\begin{tabular}{l|l|l|}
\cline{2-3}
                       & Time & Space \\ \hline
\multicolumn{1}{|l|}{KRR Optimal}  & $\mathcal{O}(n^3)$ &  $\mathcal{O}(n^2)$\\ \hline
\multicolumn{1}{|l|}{KRR Approx. } & $\mathcal{O}(q^2n)$  &  $\mathcal{O}(qn)$\\ \hline
\multicolumn{1}{|l|}{SVD Optimal}  & $\mathcal{O}((n+m)^3)$ &  $\mathcal{O}((n+m)^2)$\\ \hline
\multicolumn{1}{|l|}{SVD Approx. } & $\mathcal{O}(p^2(n+m))$  &  $\mathcal{O}(p(n+m))$\\ \hline
\end{tabular}
    \caption{OEL training complexity}
    \label{tab:training_complexity}
\end{table}

\paragraph{Decoding computational complexity} Using an output kernel makes the decoding complexity of OKR approaches costly. The cost of one prediction is dominated by the computation of the $\langle\hhatpsi(x_{te}),\, \psi(y_c)\rangle_{\bmH_y}, \forall y_c \in \bmY$, in $\bmO(n \times |\bmY|)$. Note that $|\bmY|$ is typically very big in structured prediction, for instance, in multilabel classification with $d$ labels $|\bmY| = |\{0,1\}^d| = 2^d$. However, when using OEL this decoding step is alleviated by the finite output dimension $p < n$. Indeed, by developing the norm in \eqref{eq:h-solution}, OEL decoding boils down to computing $\langle \hat G \hhatpsi(x_{te}),\, \hat G\psi(y_c)\rangle_{\R^p}, \forall y_c \in \bmY$, and the complexity is $\bmO(p \times |\bmY|)$.

\section{THEORETICAL ANALYSIS}


From a statistical viewpoint we 
are interested in controlling the {\em expected risk} of the estimator $f = d \circ \tilde g \circ h$, that for the considered loss corresponds to 
$$
\mathcal{R}(f) = \E_{X,Y}[\|\psi(f(X)) - \psi(Y)\|^2_{\bmHy}].
$$
Interpreting the decoding step in the context of structured prediction, we can leverage the so called {\em comparison inequality} from \cite{Ciliberto2016}. This inequality is applied for our case in the next lemma and relates the excess-risk of $f = d \circ \tilde g \circ h $ to the $L^2$ distance of $\tilde g \circ \hat h$ to $h^*_{\psi}$ (see \cite{Ciliberto2016} for more details on structured prediction and the comparison inequality).
\begin{restatable}{lemma}{cil}
\label{lm:comparison}
For every measurable $\hat h: \bmX \rightarrow \R^p, \tilde g: \R^p \rightarrow \bmHy$, $f = d \circ \tilde g \circ \hat h$, 
\begin{align*}
    \mathcal{R}(f) - \mathcal{R}(f^*) \leq c(\psi) ~ \sqrt{\E_X\|\tilde g \circ \hat h(X) - h^*_{\psi}(X) \|^2_{\bmHy}}
\end{align*}
with $c(\psi) = 2 \sqrt{2Q^2 + Q^4 + 1}$, $Q = \sup_{y \in \bmY} \|\psi(y)\|^2$.
\end{restatable}
It is possible to apply the comparison inequality, since the considered loss function $\Delta(y,y') = \|\psi(y) - \psi(y')\|^2_{\bmHy}$ belongs to the wide family of SELF losses \citep{Ciliberto2016} for which the comparison inequality holds. A loss is SELF if it satisfies the {\em implicit embedding property} \citep{ciliberto2020general}, i.e. there exists an Hilbert space ${\cal V}$ and two feature maps $\gamma, \theta: {\cal Y} \to {\cal V}$ such that
$$\Delta(y,y') = \left\langle\gamma(y),\theta(y')\right\rangle_{\cal V}, \quad \forall ~ y,y' \in {\cal Y}.$$
In our case the construction is direct and corresponds to ${\cal V} = {\bmH_y} \oplus \R \oplus \R$, $\gamma(y) = (\sqrt{2} \psi(y), \|\psi(y)\|_{\bmHy}^2, 1)$ and $\theta(y') = (-\sqrt{2} \psi(y'), 1, \|\psi(y')\|_{\bmHy}^2)$.

Intuitively, the idea of output embedding learning, is to find a new embedding that provides an easier regression task while being still able to predict in the initial regression space. In our formulation this is possible due to introduction of $\tilde g$ and a suitable choice of $\bmG, \tilde \bmG$. In this construction, with the orthogonal decoding $ \tilde g: z \in \R^p \rightarrow \tilde G z \in \bmHy $ described in \ref{subsec:oel_linear}, the initial regression problem $\E[\|\tilde g \circ \hat h_g(X) - h^*_{\psi}(X) \|^2]$ decomposes into two parts:
\begin{equation}\label{eq:risk-decomposition}
\begin{split}
& \underbrace{\E_X[\|\tilde g \circ \hat h_{g}(X) - \tilde g \circ h^*_{g}(X) \|^2]}_{\text{simplified regression problem}}\\ &\qquad\qquad + \underbrace{\E_X[\|\tilde g \circ  h^*_{g}(X) - h^*_{\psi}(X) \|^2]}_{\text{reconstruction error}}
\end{split}
\end{equation}
In the case of KRR, and arbitrary set of encoding function $\bmG$ from $ \R^p $ to $\bmHy$, the left term expresses as KRR excess-risk on a linear subspace of $\bmH_y$ of dimension $p$ . For the right term, defining the covariance $M_c:\bmHy \to \bmHy$ for all $ c \in [0,1]$, 
\[M_{c} = c\E_{x}(h^*(x) \otimes h^*(x)) + (1-c) \E_{y}(\psi(y) \otimes \psi(y))\]
we have the following bound due to Jensen inequality,
\begin{restatable}{lemma}{upb}\label{jensen} Under the assumptions of Lemma~\ref{lm:comparison}, when $\tilde{g}$ is a linear projection, we have
\begin{align*}
    \E_x[\|\tilde g (h^*_{g}(x)) - h^*_{\psi}(x) \|^2] &\leq 
    \langle I - P,\, M_c \rangle_{\bmH_y \otimes \bmH_y}
\end{align*}
\end{restatable}
The closer $ c $ is to $ 1 $, the tighter is the bound, but having $ c $ close to $ 0 $ could lead to a much easier learning objective. Relying on results on subspace learning in \cite{Rudi2013OnTS} and following their proofs, we bound this upper bound and get the Theorem~\ref{thm:excess-risk}. In particular, we use natural assumption on the spectral properties of the mixed covariance operator $ M_c $ as introduced in \cite{Rudi2013OnTS}. 

\noindent{\bf Assumption 1.}{\em There exist $\omega, \Omega > 0$ and $r > 1$ such that the eigendecomposition of the positive operator $M_c$ has the following form
\begin{align}\label{eq:eig-decay}
M_c = \sum_{j \in \N} \sigma_j u_j \otimes u_j, \quad \omega j^{-r} \leq \sigma_j \leq \Omega j^{-r}.
\end{align}
}
The assumption above controls the so called {\em degrees of freedom} of the learning problem. A fast decay of $\sigma_j$ can be interpreted as a problem that is well approximated by just learning the first few eigenvectors (see \cite{caponnetto2007optimal,Rudi2013OnTS} for more details).
To conclude, the first part of the r.h.s. of \eqref{eq:risk-decomposition} is further decomposed and then bounded via Lemma~18 of \cite{Ciliberto2016}, leading to Theorem~\ref{thm:excess-risk}. Before we need an assumption on the approximability of $h^*_\psi$. 

\noindent{\bf Assumption 2.}{\em 
The function $h^*_\psi$ satisfies $h^*_\psi \in \bmHy \otimes \bmHx $.
}

The assumption above where $\bmH_x$ is the RKHS associated to $k_x$, guarantees that $h^*_\psi$ is approximable by kernel ridge regression with the given choice of the kernel on the input. The kernel $K(x,x')= Id_{\bmHy} \kerx(x,x')$ satisfies this assumption.
Now we are ready to state the theorem.
%

\begin{restatable}[Excess-risk bound, KRR + OEL]{theorem}{thmer}
\label{thm:excess-risk}
Let $ \rho $ be a distribution over $ \bmX \times \bmY $, $ \rho_y $ the marginal of $Y$, $ (x_i, y_i)_{i=1}^n $ be  i.i.d samples from $ \rho $, $ (y_i)_{i=1}^m $ i.i.d samples from $\rho_y$, $ \lambda \leq \kappa^2 := \sup\limits_{x \in \bmX} \|\kerx(x, .)\|_{\bmHx}, \delta > 0$, and $M_{c}$ and $h^*_\psi$ satisfy Assumption 1 and 2. When
\begin{equation}
\label{eq:tradeoff-p}
p^r \leq  \min\left\{\frac{\omega m}{9 (1-c) \log(m/\delta)},\frac{\omega}{8 c t_n}\right\},
\end{equation}
then the following holds with probability at least $1 - 4\delta$,
\begin{equation*}
\begin{split}
    &\sqrt{\E_x(\|\tilde g \circ \hat h_g(x) - h^*_{\psi}(x) \|^2) } \leq \\&\underbrace{C \frac{Q + R}{\sqrt{\lambda n}}\log^2\tfrac{10}{\delta} + R \sqrt{\lambda}}_{\text{KRR on subspace of dim. p}} \qquad + \underbrace{\sqrt{\frac{\Omega'}{p^{r-1}}}}_{\text{Reconstruction error}}
\end{split}
\end{equation*}
where $Q= \sup_{y \in \bmY} \|k(y, .)\|_{\bmHy}$, $R = \|h^*_\psi\|_{\bmHx \times \bmHy}$,~$C = 4\kappa(1 + (4\kappa^2/\sqrt{\lambda n})^{1/2})$,~ $t_n = 5 \max\left(4c^2v_n, 2c v_n, 4cRw_n u_n, 2c w_n^2u_n, 2R^2 u_n\right)$, with $u_n = \frac{4\kappa^2 \log\frac{2}{\delta}}{\sqrt{n}}, v_n = \E_x(\|\hat h_{\psi}(x) - h^*_\psi\|^2)$, $w_n=(\frac{Q}{\lambda \kappa} u_n (1+R) + \lambda R)$. Finally $\Omega' = \Omega q_r$, $q_r$ a constant depending only on $r$ (defined in the proof).
\end{restatable}

The first term in the above bound is the usual bias-variance trade-off that we expect from KRR (c.f. \cite{caponnetto2007optimal}). The second term is an approximation error due to the projection. 

We further see a trade-off in the choice of $c$ when we try to maximize $p$. Choosing $c$ close to one aims to estimate the linear subspace of the $ h^*(x) = \E_{y|x}(\psi(y))$ which is smaller than the one of the $ \psi(y) $ leading to a better eigenvalues decay rate, but the learning is limited by the convergence of the least-square estimator as is clear by the term $v_n$ in Eq.~\eqref{eq:tradeoff-p} (via $ t_n $). Choosing $c$ close to zero leads to completely unsupervised learning of the output embeddings with a smaller eigenvalue decay rate $r$.

In the following corollary, we give a simplified version of the bound (we denote by $a(n) \gtrsim b(n)$ the fact that there exists $C >0$ such that $a(n) \geq C b(n)$ for $n \in \mathbb{N}$).
\begin{restatable}{corollary}{coro}
\label{coro}
Under the same assumptions of Theorem~\ref{thm:excess-risk}, let $\lambda = 1 / \sqrt{n}$. Then, running the proposed algorithm with a number of components 
\[p \gtrsim n^\frac{1}{2r},\] 
is enough to achieve an excess-risk of $\mathcal{R}(\widehat{f}) - \mathcal{R}(f^*)  = O(n^{-(1-\frac{1}{r})/4})$.
\end{restatable}
Note that $n^{-1/4}$ is the typical rate for problems of structured prediction without further assumptions on the problem \cite{Ciliberto2016,ciliberto2020general}. Here it is interesting to note that the more regular the problem is, the more $\frac{r-1}{r} \approx 1$, and we can almost achieve $n^{-1/4}$ rate with a number of components $p \ll n$, in particular $p = n^{1/2r}$, leading to a relevant improvement in terms of computational complexity. The decoding time complexity for one test point reduces from $\bmO(n |\bmY|)$ to $\bmO(n^{1/2r}|\bmY|)$.

Finally we note that while from an approximation viewpoint the largest $p$ would lead to better results, there is still a trade-off with the computational aspects, since an increased $p$ leads to greater computational complexity. Moreover, we expect to find a more subtle trade-off between the KRR error and the approximation error due to projection, since reducing the dimensionality of the output space have a beneficial impact on the degrees of freedom of the problem. We observed this effect from an experimental viewpoint and we expect to observe using a more refined analysis, that we leave as future work. We want to conclude with a remark on why projecting on $M_c$ with $c \in (0,1)$ should be more beneficial than just projecting on the subspace spanned by the covariance operator associated to $\psi(y)$. 

\begin{remark}[Supervised learning benefit]
Learning the embedding in a supervised fashion is interesting as the principal components of the $ \psi(y) $ may differ from that of the $h_{\psi}^*(x)$. The most basic example: $x \in \R, \psi(y) \in \R^2, x \sim \mathcal{N}(0, \sigma_x^2)$, with the relationship between $ x $ and $\psi(y) = (x,z)$ where $z \sim \mathcal{N}(0, \sigma_z^2)$ independent from $x$ with $\sigma_z^2 > \sigma_x^2$.
In this case unsupervised learning is not able to find the $ 1$-dimensional subspace where the $h^*_{\psi}(x)$ lie, whereas, supervised learning could do so, assuming $ \hat h_{\psi} $ is a good estimation of $h^*_{\psi}$. From this elementary example we could build more complex and realistic ones, by adding any kind of non-isotropic noise $ z $ that change the shape of the $h_{\psi}^*(x)$ subspace. For instance, defining eigenvalue decay of $h_{\psi}^*(x)$ and $ z \in \bmHy $ with $ z $ lower eigenvalue decay, then for $p$ sufficiently large the order of principal components start to change (comparing $h_{\psi}^*(x)$ and $\psi(y) := h_{\psi}^*(x) + z$).
\end{remark}







\section{EXPERIMENTS}\label{sec:exp}

OEL is empirically studied on image reconstruction, multilabel classification, and labeled graph prediction. More details about experiments (hyperparameter selection, dataset splitting, competitors, computational time) are provided in the Supplements, together with an exhaustive study on label ranking.
\begin{table}[!ht]
  \centering
  \begin{tabular}{lll}
    \toprule
    Notation & $c$ & $m$ \\
    \midrule
     OEL$_0$ & 1 & 0 \\
     OEL  & $\in [0,1]$ & $\ge 0$\\
    \bottomrule
  \end{tabular}
\caption{OEL notations}
\label{notations}
\end{table}
We consider two variants of our method using the notations defined in Table \ref{notations} to explicitly mention when we provide $m >  0$ additional output data (OEL) or when we use neither the reconstruction error nor unlabeled data (OEL$_0$).


\subsection{Image reconstruction}

In the image reconstruction problem provided by \citet{weston2003kernel}, the goal is to predict the bottom half of a USPS handwritten postal digit (16 x 16 pixels), given its top half. The dataset contains $7291$ labeled images and $2007$ test images. We split the $7291$ training data $(x_i, y_i)$ in $1000$ couples $(x_i, y_i)_{i=1}^{n}$ and $6000$ alone outputs $(y_i)_{i=1}^{m}$ in order to evaluate the impact of additional unexploited unsupervised data. For all tested methods, we used all the 7000 output training data as candidate set for the decoding. \paragraph{Role of $\lambda$ and $p$.} In Fig. \ref{usps_mse} we present  for a fixed value of $c=0.15$\footnote{value selected using a 5 repeated random sub-sampling validation (80\%/20\%)}, the behaviour of OEL in terms of Mean Squared Error of $G^*G\hhatpsi$ ~w.r.t. $\lambda$ and $p$, compared to IOKR (chosen as the baseline). We observe that the minimum of all measured MSEs is attained by OEL, highlighting the interest of finding a good trade-off between the two  regularization parameters. 
\begin{figure}[h!]
    \centering
    \includegraphics[width=75mm]{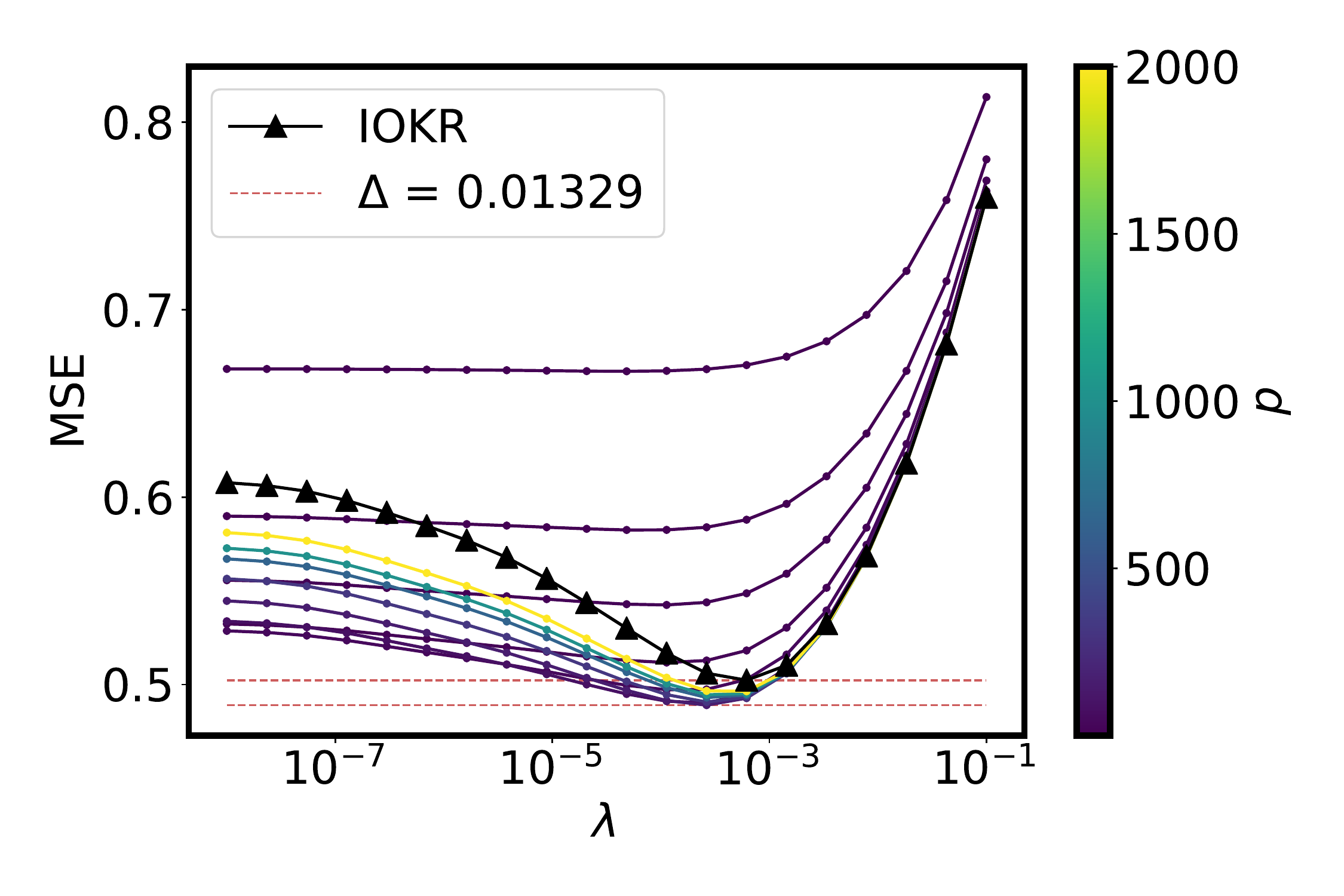}
    \caption{Test MSE of OEL w.r.t $ \lambda $ and $ p$ compared to IOKR (black plot line) on the USPS problem.}\label{usps_mse}
\end{figure}
\paragraph{Comparison with SOTA methods.} We then compared OEL to state-of-the-art methods: SPEN \citep{belanger2016}, IOKR \citep{brouard2016input}, and Kernel Dependency Estimation (KDE) \citep{weston2003kernel}. For SPEN usually exploited for multi-label classification, we employed the standard architecture and training method described in the corresponding article (cf. supplements for more details). The hyper-parameters for all methods (including $\lambda, p, c$ for OEL, and SPEN layers' sizes) have been selected using 5 repeated random sub-sampling validation (80\%/20\%). We evaluated the results in term of RBF loss (e.g. Gaussian kernel loss), the relevance of this loss on this problem has been shown in \citet{weston2003kernel}.
The obtained results are given in Table \ref{usps_table1}. Firstly, we see that SPEN obtains worse results than KDE, IOKR, and OEL. Indeed, the problem dimensions is typically favourable for kernel methods: a small dataset ($n=1000$) but a quite complex task with big input/output dimensions ($d=128$). Despite here the artificial split of the dataset for method analysis purpose, we claim that this is a classic situation in structured prediction (cf. real-world problem of metabolite identification below). Furthermore, note that the number of hyperparameters for SPEN (architecture and optimization) is usually larger than OEL. Secondly, we see that OEL$_0$, without additional data, obtains improved results in comparison to IOKR and KDE. When adding the $m=6000$ output data, OEL further improves the results. This shows the relevance of the learned embeddings, which can take advantage both of a low-rank assumption, and unexploited unsupervised output data.

\begin{table}[ht!]
    \centering
    \begin{tabular}{lll}
    \toprule
    Method     &  RBF loss & p \\
    \midrule
    SPEN  & 0.801 $\pm$ 0.011 & 128\\
    KDE  & 0.764 $\pm$ 0.011& 64\\
    IOKR & 0.751 $\pm$ 0.011 & $\infty$\\
    OEL$_0$ & 0.734 $\pm$ 0.011& 64 \\
    OEL & {\bf 0.725 $\pm$ 0.011}&98\\
    \bottomrule
    \end{tabular}
    \caption{Test mean losses and standard errors for IOKR , OEL, KDE , and SPEN  on the USPS digits reconstruction problem where $n/m = 1000/6000$.}\label{usps_table1}
\end{table}

\subsection{Multi-label classification}
In this subsection we evaluate the performances of OEL on different benchmark multi-label datasets described in Table~\ref{multilabel_dataset_1}.
\begin{table}[!ht]
  \centering
  \begin{tabular}{llllll}
    \toprule
    Dataset & $n$ & $n_{te}$ & $n_{features}$ & $n_{labels}$\\
    \midrule
     Bibtex & 4880 & 2515 & 1836 & 159 \\
     Bookmarks & 60000 & 27856 & 2150 & 208 \\
     Corel5k & 4500 & 499 & 37152 & 260 \\
    \bottomrule
  \end{tabular}
\caption{Multi-label datasets description.}
\label{multilabel_dataset_1}
\end{table}

\paragraph{Small training data regime.} In a first experiment we compared OEL with IOKR in a setting where only a small number of training examples is known and unsupervised output data are available. For this setting, we split the multi-label datasets using a smaller training set and using the rest of the examples as unsupervised output data. For IOKR and OEL, we used Gaussian kernels for both input and output. For OEL, hyper-parameters have been selected using 5 repeated random sub-sampling validation (80\%/20\%) and the same hyper-parameters were used for IOKR due to expensive computation. 
The results of this comparison are given in Table~\ref{tab:multilabel_small}. We observe that OEL$_0$ obtains higher $F_1$ scores than IOKR in this setup. Using additional unsupervised data provides further improvement in the case of the Bookmarks and Corel5k datasets. This highlights the interest of OEL when typically the supervised dataset is small in comparison to the difficulty of the task, and unexploited output data are available.

\begin{table}[!ht]
  \centering
    \begin{tabular}{llll}
    \toprule
    &  Bibtex & Bookmarks & Corel5k \\
    \midrule
 $n$ & $2000$ & $2000$ & $2000$\\
  $m$ & $2880$ & $4000$ & $2500$\\
  $n_{te}$ & $2515$ & $2500$ & $499$\\
    \midrule
    IOKR &  35.9 & 22.9 & 13.7  \\
    OEL$_0$ & {\bf 39.7} & 25.9 & 16.1 \\
    OEL & {\bf 39.7} & {\bf 27.1} & {\bf 19.0}\\
    \bottomrule
  \end{tabular}
    \caption{Test $F_1$ score of OEL and IOKR on different multi-label problems in a small training data regime.}
    \label{tab:multilabel_small}
\end{table}

We further show the impact of additional unsupervised data on the Bookmarks dataset by training the KRR with only $n = 2000$ data, and training OEL with various numbers of unexploited data from $0$ to $50000$ randomly selected. Figure \ref{fig:bookmark_unexp} shows that adding unsupervised output data through the right term of Equation \eqref{eq:linear-emp-3} allows to improve the results up to a certain level.

\begin{figure}[htb!]
    \centering
    \includegraphics[width=67mm]{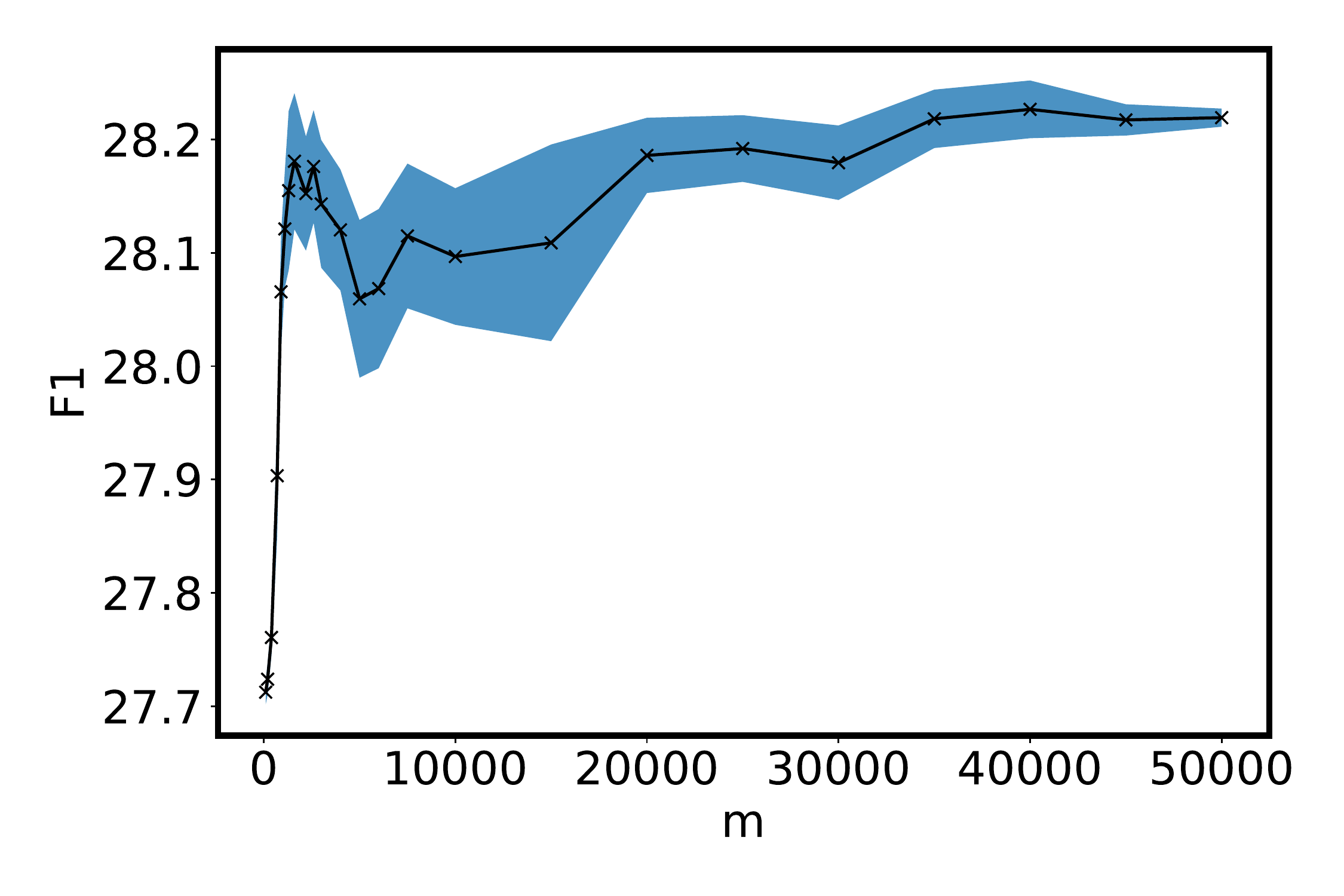}
    \caption{Test F1 of OEL on Bookmarks dataset with $n/n_{te}=2000/27856$ w.r.t the quantity of randomly selected unsupervised data $m \in [0, 50000]$ used.}\label{fig:bookmark_unexp}
\end{figure}


\paragraph{Learning without unlabeled data $m=0$.} In a second experiment we considered the case where all training data are used for training OEL and no unexploited data are available to use as unsupervised data. This allows us to compare OEL$_0$ with several multi-label and structured prediction approaches including IOKR \citep{brouard2016input}, logistic regression (LR) trained independently for each label \citep{lin2014}, a two-layer neural network with cross entropy loss (NN) by \citep{belanger2016}, the multi-label approach PRLR (Posterior-Regularized Low-Rank) \citep{lin2014}, the energy-based model SPEN (Structured Prediction Energy Networks) \citep{belanger2016} as well as DVN (Deep Value Networks) \citep{gygli2017}. 
The results in Table~\ref{tab:multilabel_comparison} show that OEL$_0$ can compete with state-of-the-art dedicated multilabel methods on the standard datasets Bibtex and Bookmarks. With Bookmarks ($n/n_{te}=60000/27856$) we used a Nyström approximation with 15000 anchors for IOKR and OEL$_0$ to reduce the training complexity. In order to alleviate the decoding complexity, 
we learned OEL$_0$ with a subset of the training data used for $\hat{h}_\psi$, containing only $12000$ training datapoints: IOKR decoding took about 56 minutes, and OEL$_0$ decoding less than 4 minutes. With a drastically smaller amount of time, OEL$_0$ achieves the same order of magnitude of $F_1$ as IOKR at a lower cost and still has better performance than all other competitors.

\begin{table}[!ht]
  \centering
  \begin{tabular}{lll}
    \toprule
    Method     &  Bibtex & Bookmarks \\
    \midrule
    IOKR & 44.0 & {\bf39.3}\\
    OEL$_0$  & 43.8  & 39.1\\
    LR  & 37.2 & 30.7\\
    NN   & 38.9 & 33.8\\
    SPEN & 42.2 & 34.4\\
    PRLR & 44.2 & 34.9\\
    DVN  & {\bf 44.7} & 37.1\\
    \bottomrule
  \end{tabular}
    \caption{Tag prediction from text data. $F_1$ score of OEL$_0$ compared to state-of-the-art methods. LR \citep{lin2014}, NN \citep{belanger2016}, SPEN \citep{belanger2016}, PRLR \citep{lin2014}, DVN \citep{gygli2017}. Results are taken from the corresponding articles.}
    \label{tab:multilabel_comparison}
\end{table}



\subsection{Metabolite Identification}
\paragraph{Learning with a very large $\vert \bmY \vert$ and a large $m$}
In this subsection, we apply OEL to the metabolite identification problem, which is a difficult problem characterized by a high output dimension $d$ and a very large number of candidates ($\vert \bmY \vert$). 
The goal of this problem is to predict the molecular structure of a metabolite given its tandem mass spectrum. The molecular structures of the metabolites are represented by fingerprints, that are binary vectors of length $d = 7593$. Each value of the fingerprint indicates the presence or absence of a certain molecular property. Labeled data are expensive to obtain, despite the problem complexity we only have $n=6974$ labeled data. However, an output set $\bmY$ containing more than 6 millions of fingerprints is available. State-of-the-art results for this problem have been obtained with the IOKR method by \citet{brouard2016fast}, and we adopt here a similar numerical experimental protocol (5-CV Outer/4-CV Inner loops), probability product input kernel for mass spectra, and Gaussian-Tanimoto output kernel on the molecular fingerprints. Using randomized singular value decomposition we trained OEL with $m=10^5$ molecular fingerprints, which are not exploited with plain IOKR as the corresponding inputs (spectra) are not known. Details can be found in the Supplementary material.
Analyzing the results in Table \ref{metabolite_table}, we observe that $OEL$ improved upon plain IOKR. Such accuracy improvement is crucial in this real-world task. The selected balancing parameter by a inner cross-validation on training set is $\hat{c} = 0.75$ in average on the outer splits, imposing a balance between the influence of the small size labeled dataset and the large unsupervised output set. Again using an output kernel reveals to be particularly efficient in supervised problems with complex outputs and a small training labeled dataset. 

\begin{table}[htb!]
  \centering
  \begin{tabular}{llll}
    \toprule
    Method       & Gaussian  & Top-k accuracies\\
    & -Tanimoto loss &  $k=1$ | $k=5$ | $k=10$\\
    \midrule
    SPEN &  $0.537 \pm 0.008$ & $25.9 \%\,|\, 64.3 \%\,|\, 54.1 \%$ \\
    IOKR  & $0.463 \pm 0.009$ &  $29.6 \%\,|\, 61.1 \%\,|\, 71.0 \%$  \\
    OEL & $\mathbf{0.441 \pm 0.009}$  & $\mathbf{31.2} \%\,|\, \mathbf{63.5} \%\,|\, \mathbf{72.7} \%$  \\
    \bottomrule
  \end{tabular}
    \caption{Test mean losses and standard errors for the metabolite identification problem.}
    \label{metabolite_table}
\end{table}

\section{CONCLUSION}\label{sec:conclusion}
Within the context of Output Kernel Regression for structured prediction, we propose a novel general framework OEL that approximates the infinite dimensional given embedding by a finite one by exploiting labeled training data and output data. Developed for linear projections, OEL leverages the rich Hilbert space representations of structured outputs while controlling the excess risk of the model through supervised dimensionality reduction in the output space, as witnessed by the theoretical analysis.
Our empirical experiments demonstrate that the method can take advantage of additional output data. Moreover, experimental results show a state-of-the-art performance or exceeding it for various applications with a drastically reduced decoding time compared to IOKR.

\newpage
\bibliographystyle{unsrtnat}
\bibliography{references}
\nocite{*}

\newpage
\onecolumn
\section{Supplementary Materials}\label{sec:supplements}

This supplementary material is organized as follows. Subsection \ref{sec: sup_def} introduces definitions and notations that will be useful in the subsections \ref{sec: sup_lemma} and \ref{sec: sup_thm}. In subsection \ref{sec: sup_lemma} we prove a set of lemmas necessary for proving the excess-risk theorem in \ref{sec: sup_thm}. In subsection \ref{sec: sup_time} we give details and experimental results on the time complexity of OEL in comparison with IOKR. In subsection \ref{sec: sup_exp} we give details about the experiments of section \ref{sec:exp}. In subsection \ref{ranking} we give additional experimental results on label ranking. 

\subsection{Notations and definitions}\label{sec: sup_def}

Here we introduce, and give basic properties, on the ideal and empirical linear operators that we will use in the following to prove the excess-risk theorem.

\begin{itemize}
    \item $\phi: \bmX \rightarrow \bmHx$, $\forall x \in \bmX$, $\phi(x) = k_x(x, .) $
    \item $S : f \in \bmHx \rightarrow \langle f, \kerx(x,.) \rangle_{\bmHx} \in L^2(\bmX, \rho_\mathcal{X})$
    \item $Z : h \in \bmHy \rightarrow \langle h, h^*(.) \rangle_{\bmHy} \in L^2(\bmX, \rho_\mathcal{X})$, (with $h^*(x):= \E_{y|x}(\psi(y))$)
    \item $S_n : f \in \bmHx \rightarrow \frac{1}{\sqrt{n}} (\langle f, \kerx(x_i,.) \rangle_{\bmHx})_{1 \leq i \leq n} \in \R^n$
    \item $Z_n : h \in \bmHy \rightarrow \frac{1}{\sqrt{n}} (\langle h, \psi(y_i) \rangle_{\bmHy}))_{1 \leq i \leq n} \in \R^n$
    \item $C = \E_x(\phi(x) \otimes \phi(x)) = S^*S$ and its empirical counterpart $C_n = \frac{1}{n}\sum\limits_{i=1}^n \phi(x_i) \otimes \phi(x_i)$
    \item $V = \E_y(\psi(y) \otimes \psi(y))$ and its empirical counterpart $V_n = \frac{1}{m}\sum\limits_{i=1}^m \psi(y_i) \otimes \psi(y_i)$
    \item $N = \E_x(h^*(x) \otimes h^*(x)) = HCH^* $ (with $H = Z^*SC^{\dagger}$, via assumption 2, cf. lemma 16 in \citet{Ciliberto2016}) and its empirical counterpart $N_n = H_nC_nH_n^* $ ( with $H_n = Z_n^*S_n(C_n + \lambda I)^{-1}$,  cf. lemma 17 in \citet{Ciliberto2016}).
    \item $ M_c = c N + (1-c)V $ and its empirical counterpart $\hat M_c = c N_n + (1-c)V_m$
    \item $h_{\psi}^*(.) = H\phi(.)$ and its empirical counterpart $\hat h_{\psi}(.) = H_n\phi(.)$
    \item The gamma function $\Gamma :  z \rightarrow \int_{0}^{\infty} t^{z-1}e^{-t}dt$
\end{itemize}

We have the following properties:

\begin{itemize}
    \item $Z^*S = \int\limits_{\mathcal{X} \times \mathcal{Y}} \psi(y) \otimes \phi(x) d\rho(x, y)$
    \item $Z_n^* S_n = \frac{1}{n}\sum\limits_{i=1}^n \psi(y_i) \otimes \phi(x_i)$
    \item If $h^*(x):= \E_{y|x}(\psi(y)) \in \mathcal{H} $, then $h^*(x) = H\phi(x), \quad \forall x \in \mathcal{X} $, with $ H = Z^*SC^\dagger \in \bmHy \otimes \bmHx$, via assumption 2, cf. lemma 16 in \citet{Ciliberto2016}.
    \item $\forall \lambda > 0 $, $\hat h(x) = H_n\phi(x), \quad \forall x \in \bmX $, with $ H_n = Z_n^*S_n (C_n + \lambda I)^{-1} \in \bmHy \otimes \bmHx $, cf. lemma 17 in \citet{Ciliberto2016})
\end{itemize}
 
 \subsection{Lemmas}\label{sec: sup_lemma}

First, we give the following lemma showing the equivalence in the linear case between the general initial objective \eqref{eq:oel-gen} and the mixed linear subspace estimation one \eqref{eq:linear-pb-3}.
\begin{lemma}

Using the solution of the $\psi(Y)$ regression problem, Eq. \eqref{eq:linear-pb-1} is equivalent to: 
 \begin{equation}
\min\limits_{G} \gamma \E_{X} [\| G^*G \hbpsi(X) -  \hbpsi(X)\|^2_{\bmHy}] + (1- 2\gamma) \E_Y [\| G^*G\psi(Y) - \psi(Y)\|^2_{\bmHy}].
 \end{equation}
\end{lemma}

\begin{proof}In the linear case, \eqref{eq:oel-gen} instantiates as
\begin{equation}\label{eq:linear-pb-2}
 \min\limits_{G} \gamma \E_{X, Y} [\| G (\hbpsi(X) - \psi(Y)) \|^2_{\R^p}] + (1- \gamma) \E_Y [\| G^*G\psi(Y) - \psi(Y)\|^2_{\bmHy}].
 \end{equation}
Decomposing the first term, with $\hbpsi(x) = \E_{Y|X=x}[\psi(Y)]$, and noticing that $\| G \psi(y)\|_{\R^p} = \| G^*G \psi(y)\|_{\bmHy}$ ($G$ has orthogonal rows), one can check that we obtain the desired result.
\end{proof}

From here, we give the lemmas, and their proofs, that we used in order to prove the main theorem in the next section.

First, we leverage the {\em comparison inequality} from \cite{Ciliberto2016}, allowing to relate the excess-risk of $f = d \circ \tilde g \circ h $ to the $L^2$ distance of $\tilde g \circ \hat h$ to $h^*_{\psi}$.
 
\cil*

\begin{proof} The considered loss function $\Delta(y,y') = \|\psi(y) - \psi(y')\|^2_{\bmHy}$ belongs to the wide family of SELF losses \citep{Ciliberto2016} for which the comparison inequality holds. A loss is SELF if it satisfies the {\em implicit embedding property} \citep{ciliberto2020general}, i.e. there exists an Hilbert space ${\cal V}$ and two feature maps $\gamma, \theta: {\cal Y} \to {\cal V}$ such that
\[\Delta(y,y') = \left\langle\gamma(y),\theta(y')\right\rangle_{\cal V}, \quad \forall ~ y,y' \in {\cal Y}.\]
In our case the construction is direct and corresponds to ${\cal V} = {\bmH_y} \oplus \R \oplus \R$, $\gamma(y) = (\sqrt{2} \psi(y), \|\psi(y)\|_{\bmHy}^2, 1)$ and $\theta(y') = (-\sqrt{2} \psi(y'), 1, \|\psi(y')\|_{\bmHy}^2)$

Hence, directly applying Theorem 3. from \cite{ciliberto2020general}, we get a constant:

\[c_{\Delta} = 2 \times \sup_{y \in \bmY} \|\gamma(y)\| = 2\sup_{y \in \bmY} \sqrt{2\|\psi(y)\|^2 + \|\psi(y)\|^4 + 1} = 2 \sqrt{2Q^2 + Q^4 + 1}\]

\end{proof}

In the theorem's proof, we will split the surrogate excess-risk of $\tilde g \circ \hat h$, using triangle inequality, in two terms: 1) the KRR excess-risk of $ \hat h $ in estimating the new embedding $g$, 2) the excess-risk or reconstruction error of the learned couple $(g, \tilde g)$ when recovering $\psi$. For now, $ g, \tilde g $ are supposed to be of the form: $g : y \rightarrow G\psi(y), \tilde g: z \rightarrow G
^*z$, with $G \in \R^p \otimes \bmH_y$ such that $GG^* = I_{\R^p}$. 

The following lemma give a bound for the first term: the KRR excess-risk on the learned linear subspace of dimension $p$. To do so, we simply bound it with the KRR excess-risk on the entire space $\bmH_y$, and leave as a future work a refined analysis of this term studying its dependency w.r.t $p$.

\begin{lemma}[Kernel Ridge Excess-risk Bound on a linear subspace of dimension $p$]\label{krr} Let be $\hat G \in \R^p \otimes \bmH_y$ such that $\hat G \hat G^* = I_{\R^p}$, then with probability at least $1-\delta$:

\[\sqrt{\E_x(\|\hat G^*\hat G(\hat h_{\psi}(x) - h^*_{\psi}(x))\|^2)} \leq C \frac{Q + R}{\sqrt{\lambda n}}\log^2\tfrac{10}{\delta} + R \sqrt{\lambda}\]

with $Q:= \sup\limits_{y \in \bmY} \|\psi(y)\|_{\bmHy}$, $R = \|h^*_\psi\|_{\bmHx \times \bmHy}$, $C = 4\kappa(1 + (4\kappa^2/\sqrt{\lambda n})^{1/2})$.

\end{lemma}

\begin{proof}

We observe that: $\sqrt{\E_x(\|\hat G^*\hat G(\hat h(x)_{\psi} - h^*_{\psi}(x))\|^2)} \leq \sqrt{\E_x(\|\hat h(x)_{\psi} - h^*(x)_{\psi}\|^2)}$

Then, considering the assumptions of Theorem \ref{thm:excess-risk}, we use result for kernel ridge regression from \cite{Ciliberto2016}, with $Q:= \sup\limits_{y \in \bmY} \|\psi(y)\|_{\bmHy}$, $R = \|h^*_\psi\|_{\bmHx \times \bmHy}$, $C = 4\kappa(1 + (4\kappa^2/\sqrt{\lambda n})^{1/2})$.

\end{proof}

Then, we show that we can upper bound the reconstruction error term, using Jensen inequality, by the ideal counterpart of the empirical objective w.r.t the output embedding $ G $ of the algorithm \ref{algo:linear}.

\upb*

\begin{proof} Let $c \in [0,1]$, $G \in \R^p \otimes \bmHy$ such that $GG^* = I_{\R^p}$, $g: y \rightarrow G\psi(y)$, $\tilde g: z \rightarrow G^*z$

Defining the projection $P = G^*G \in \bmHy \times \bmHy$, we have
\begin{align*}
    \E_x[\|\tilde g (h^*_{g}(x)) - h^*_{\psi}(x) \|^2] & = \E_x(\|P h^*_{\psi}(x) - h^*_{\psi}(x)\|^2)\\
    &= c \E_x(\|P h^*_{\psi}(x) - h^*_{\psi}(x)\|^2) + (1 - c)\E_x(\|P h^*_{\psi}(x) - h^*_{\psi}(x)\|^2)\\
    &= c \E_x(\| P h^*_{\psi}(x) - h^*_{\psi}(x)\|^2) + (1 - c)\E_x(\| P \E_{y|x}(\psi(y)) - \E_{y|x}(\psi(y))\|^2)\\
    &\leq c \E_x(\|P h^*(x) - h^*(x)\|^2) + (1 - c)\E_{y}(\|P \psi(y) - \psi(y)\|^2) \quad \text{(Jensen's inequality)}\\
    &= c \langle I - P,\, \E_{x}(h^*(x) \otimes h^*(x)) \rangle_{\bmHy \otimes \bmHy} + (1-c) \langle I - P,\, \E_{y}(\psi(y) \otimes \psi(y)) \rangle_{\bmHy \otimes \bmHy}\\
    &= \langle I - P,\, M_c \rangle_{\bmHy \otimes \bmHy}
\end{align*}
\end{proof}

Then, the following lemma study the efficiency of our algorithm in minimizing $\langle I - \hat G^*\hat G,\, M_c \rangle_{\bmHy \otimes \bmHy}$. To do so, we followed the approach of \cite{Rudi2013OnTS}, with the help of an additional necessary technical lemma \ref{mainlemma}.

\begin{lemma}[OEL Subspace estimation]\label{set}If $t_{min} \leq t \leq \min(\|V\|_{\infty}, \|N\|_{\infty})$, with probability $1-3\delta$:

\[\sqrt{\langle I - \hat G^*\hat G,\, M_c \rangle_{\bmHy \otimes \bmHy}} \leq 3 \Omega' t^{-1/2r} \sqrt{\sigma_k(M_c) + t}\]

with: $t_{min} = \max((1-c) t_1, c t_2)$, and $t_1 = \frac{9}{m}\log(\frac{m}{\delta}), t_2 = 5 \max\left(4c^2v_n, 2c v_n, 4cRw_n u_n, 2c w_n^2u_n, 2R^2 u_n\right)$, with $u_n = \frac{4\kappa^2 \log\frac{2}{\delta}}{\sqrt{n}}, v_n = \E_x(\|\hat h_{\psi}(x) - h^*_\psi\|^2)$, $w_n=(\frac{Q}{\lambda \kappa} u_n (1+R) + \lambda R)$, and $\Omega' = (\Omega^{1/r}\Gamma(1 - 1/r)\Gamma(1 + 1/r)\Gamma(1/r))^{1/2}$, $R =  \|h^*_{\psi}\|_{\bmH}$.
\end{lemma}

\begin{proof} Let be $c \in [0,1]$, we have from Proposition C.4. in \cite{Rudi2013OnTS}:
\begin{align}
    \langle I - \hat G^*\hat G,\, M_c \rangle_{\bmHy \otimes \bmHy} &= \|(\hat G^*\hat G - I)M_c^{\frac{1}{2}}\|_{HS}^2\label{hsd}
\end{align}

Then, following \cite{Rudi2013OnTS} proofs, we split (\ref{hsd}) into three parts, and bound each term,
\[\|(\hat G^*\hat G - I)M_c^{\frac{1}{2}}\|_{HS} \leq \underbrace{\|(M_c + tI)^{\frac{1}{2}}(\hat M_c + tI)^{-\frac{1}{2}}\|_{\infty}}_{\mathcal{A}} \times \underbrace{(\sigma_k(\hat M_c) + t)^{\frac{1}{2}}}_{\mathcal{B}} \times \underbrace{\|(M_c + tI)^{-\frac{1}{2}}M_c^{\frac{1}{2}}\|_{HS}}_{\mathcal{C}}\]

[\textbf{Bound} $\mathcal{A}=\|(M_c + tI)^{\frac{1}{2}}(\hat M_c + tI)^{-\frac{1}{2}}\|_{\infty}$]
We apply Lemma \ref{mainlemma}, which gives if $t_{min} \leq t \leq \min(\|V\|_{\infty}, \|N\|_{\infty})$. Then with probability $1-3\delta $ it is
\[\frac{2}{3} \leq \|(M_c + tI)^{\frac{1}{2}}(\hat M_c + tI)^{-\frac{1}{2}}\|_{\infty} \leq 2\]

[\textbf{Bound} $\mathcal{B}=(\sigma_k(\hat M_c) + t)^{\frac{1}{2}}$]
As in Lemma 3.5 in \cite{Rudi2013OnTS} (cf. Lemma B.2 point 4), the previous lower bound $\frac{2}{3}\sqrt{\frac{2}{3}} \leq \|(M_c + tI)^{\frac{1}{2}} (\hat M_c + tI)^{-\frac{1}{2}}\|_{\infty}$ gives us that:
\[\sqrt{\sigma_k(\hat M_c) + t} \leq \frac{3}{2} \sqrt{\sigma_k(M_c) + t}\]

[\textbf{Bound} $\mathcal{C}=\|(M_c + tI)^{-\frac{1}{2}}M_c^{\frac{1}{2}}\|_{HS}$]
Lemma 3.7 of \cite{Rudi2013OnTS} with the eigenvalue decay assumption of our theorem gives us that:
\[\|(M_c + tI)^{-\frac{1}{2}}M_c^{\frac{1}{2}}\|_{HS} \leq \Omega' t^{-\frac{1}{2r}}\]

with: $\Omega' = (\Omega^{1/r}\Gamma(1 - 1/r)\Gamma(1 + 1/r)\Gamma(1/r))^{1/2}$

Finally, we get the wanted upper bound on $\sqrt{\langle I - \hat G^*\hat G,\, M_c \rangle_{\bmHy \otimes \bmHy}}$.

\end{proof}

The following lemma is the technical lemma necessary in lemma \ref{set}'s proof to bound the 2 first terms ($\mathcal{A}$ and $\mathcal{B}$) in the decomposition of the reconstruction error.

\begin{lemma}[Term $\mathcal{A}$]\label{mainlemma} Let $t_{min} \leq t \leq \min(\|V\|_{\infty}, \|N\|_{\infty})$. Then with probability $1-3\delta $ it is

\[\frac{2}{3} \leq \|(M_c + tI)^{\frac{1}{2}} (\hat M_c + tI)^{-\frac{1}{2}}\|_{\infty} \leq 2 \]

with: $t_{min} = \max((1-c) t_1, c t_2)$, and $t_1 = \frac{9}{m}\log(\frac{m}{\delta}), t_2 = 5 \max\left(4c^2v_n, 2c v_n, 4cRw_n u_n, 2c w_n^2u_n, 2R^2 u_n\right)$, with $u_n = \frac{4\kappa^2 \log\frac{2}{\delta}}{\sqrt{n}}, v_n = \E_x(\|\hat h_{\psi}(x) - h^*_\psi\|^2)$, $w_n=(\frac{Q}{\lambda \kappa} u_n (1+R) + \lambda R), R =  \|h^*_{\psi}\|_{\bmH}$.
\end{lemma}

\begin{proof} We will decompose this term to study separately the convergence of $ V_m $ to $ V $, and the convergence of $ N_n = H_nC_nH_n^*$ to $ N = HCH^*$. However, in order to study the convergence of $ N_n $ to $ N $ we will need to do an additional decomposition as $ N $ is estimated thanks to the KRR estimation $H_n$ of $H$. So, we do the two decompositions leading to the two terms (1), (2), and then we bound each term:
\begin{align*}
    \|(M_c + tI)^{\frac{1}{2}} (\hat M_c + tI)^{-\frac{1}{2}}\|_{\infty} &= \|(c N + (1-c)V + tI)^{\frac{1}{2}} (c N_n + (1-c)V_m + tI)^{-\frac{1}{2}}\|_{\infty}\\
    &\leq \underbrace{\|(c N + (1-c)V + tI)^{\frac{1}{2}}(c N + (1-c)V_m + tI)^{-\frac{1}{2}}\|_{\infty}}_{(1)} \\&\times \underbrace{\|(c N + (1-c)V_m + tI)^{\frac{1}{2}}(c N_n + (1-c)V_m + tI)^{-\frac{1}{2}}\|_{\infty}}_{(2)}
\end{align*}

[\textbf{Bound} $(1)$] We apply lemma \ref{lemma:cov} and get, if $ \frac{9}{m}\log(\frac{m}{\delta}) \leq t \leq \|V\|_{\infty}$, with probability $1-\delta $: $\sqrt{\frac{2}{3}} \leq (1) \leq \sqrt{2}$

[\textbf{Bound} $(2)$] We write:
$\|(c H C H^* + (1-c)V_m + tI)^{\frac{1}{2}}(c H_n C_n H_n + (1-c)V_m + tI)^{-\frac{1}{2}}\|_{\infty} = \|(I-B_n)^{-1} \|_{\infty}^{1/2}$

with: $B_n = (c H C H^* + (1-c)V_m + tI)^{-\frac{1}{2}}c(H_nC_nH_n^* - HCH^*)(c H C H^* + (1-c)V_m + tI)^{-\frac{1}{2}}$ and:
\begin{align*}
    \|B_n\|_{\infty} \leq \|(c H C H^* + tI)^{-\frac{1}{2}}c(H_nC_nH_n^* - HCH^*)(c H C H^* + tI)^{-\frac{1}{2}}\|_{\infty}
\end{align*}
We apply lemma \ref{lemma:HnCnHn}, and if $t \geq 5 \max\left(4c^2v_n, 2c v_n, 4cRw_n u_n, 2c w_n^2u_n, 2R^2 u_n\right)$ with probability $1-2\delta$,
$\sqrt{\frac{2}{3}} \leq  (2) \leq \sqrt{2}$

[\textbf{Conclusion}] As $(\mathcal{A}) = (1) \times (2)$, we conclude by union bound with probability $1-3\delta$, the bound on $\mathcal{A}$: $\frac{2}{3} \leq \mathcal{A} \leq 2$

\end{proof}

The next three lemmas are technical lemmas about convergences used in the proof of lemma \ref{mainlemma}.

\begin{lemma}[Convergence of covariance operator]\label{lemma:cov} Let be $H \in \bmH_y \otimes \bmH_x$, $A=\E_x[H\phi(x) \otimes H\phi(x)]$, $A_n=\frac{1}{n} \sum\limits_{i=1}^n H\phi(x_i) \otimes H\phi(x_i)$, $B \in \bmH_y \otimes \bmH_y$ positive semi-definite, $ \frac{9}{n}\log(\frac{n}{\delta}) \leq t \leq \|A\|_{\infty}$, with probability $1-\delta$ it is
\[\sqrt{\frac{2}{3}} \leq \|(A + B + tI)^{\frac{1}{2}}(A_n + B + tI)^{-\frac{1}{2}}\|_{\infty} \leq \sqrt{2}\]
\end{lemma}
\begin{proof}

We write:
\[\|(A + B + tI)^{\frac{1}{2}}(A_n + B + tI)^{-\frac{1}{2}}\|_{\infty} = \|(I-B_n)^{-1} \|_{\infty}^{1/2}\]
with: $B_n = (A + B + tI)^{-\frac{1}{2}}(A_n - A)(A + B + tI)^{-\frac{1}{2}}$ and:
\begin{align*}
    \|B_n\|_{\infty} &= \|(A + B + tI)^{-\frac{1}{2}}(A_n - A)(A + B + tI)^{-\frac{1}{2}}\|_{\infty}\\
    &\leq \|(A + tI)^{-\frac{1}{2}}(A_n - A)(cM + tI)^{-\frac{1}{2}}\|_{\infty}
\end{align*}
We apply Lemma 3.6 of \cite{Rudi2013OnTS} and get with probability $1-\delta$, if $\frac{9}{n}\log(\frac{n}{\delta}) \leq t \leq \|A\|_{\infty}$
\[\|(A + tI)^{-\frac{1}{2}}(A_n - A)(A + tI)^{-\frac{1}{2}}\|_{\infty} \leq \frac{1}{2}\]
and we conclude as in Lemma 3.6 of \cite{Rudi2013OnTS}.

\end{proof}

\begin{lemma}[Bound $\|H_n - H\|_{\infty}$]\label{lemma:Hn} Let be $u_n=\frac{4\kappa^2\log\frac{2}{\delta}}{\sqrt{n}}$, with probability $1-2\delta $ it is

\[\|H_n - H\|_{\infty} \leq \frac{Q}{\kappa \lambda}u_n(1+\|h^*_{\psi}\|_{\bmH}) + \lambda \|h^*_{\psi}\|_{\bmH}\]

\end{lemma}

\begin{proof}

In order to bound $\|H_n - H\|_{\infty}$ we do the following decomposition in three terms, and bound each term:
\begin{align*}
    \|H_n - H\|_{\infty} &=  \|Z_n^*S_n(C_n + \lambda I)^{-1} - Z^*SC^{\dagger}\|_{\infty}\\
    &\leq \underbrace{\|(Z_n^*S_n- Z^*S)(C_n + \lambda I)^{-1}\|_{\infty}}_{(A)} + \underbrace{\| Z^*S((C_n + \lambda I)^{-1} - (C + \lambda I)^{-1})\|_{\infty}}_{(B)}\\
    &\quad + \underbrace{\|Z^*S ((C + \lambda I)^{-1} - C^{\dagger})\|_{\infty}}_{(C)}
\end{align*}
[\textbf{Bound (A)}]
We have:
\begin{align*}
    (A) &= \|(Z_n^*S_n- Z^*S)(C_n + \lambda I)^{-1}\|_{\infty}
    \leq \frac{1}{\lambda}\|Z_n^*S_n- Z^*S\|_{HS}
\end{align*}
From \cite{Ciliberto2016} (proof of lemma 18.), with probability $1-\delta$: $(A) \leq \frac{4Q\kappa\log\frac{2}{\delta}}{\lambda \sqrt{n}}$.

[\textbf{Bound (B)}]
We have:
\begin{align*}
    (B) &= \| Z^*S((C + \lambda I)^{-1} - (C_n + \lambda I)^{-1})\|_{\infty}\\
    &= \| Z^*S((C + \lambda I)^{-1}(C_n - C)(C_n + \lambda I)^{-1})\|_{\infty}\\
    &\leq \| Z^*S(C + \lambda I)^{-1}\|_{\infty} \|(C_n - C)\|_{\infty} \|(C_n + \lambda I)^{-1}\|_{\infty}\\
    &\leq \frac{1}{\lambda} \|h^*_{\psi}\|_{\bmH} \|(C_n - C)\|_{\infty}
\end{align*}
where we used the fact that for two invertible operators $A, B$: $A^{-1} - B^{-1} = A^{-1}(B-A)B^{-1}$, and noting that $\| Z^*S(C + \lambda I)^{-1}\|_{\infty} \leq \| Z^*S(C + \lambda I)^{-1}\|_{HS} \leq \|H\|_{HS} = \|h^*_{\psi}\|_{\bmH}$. From \cite{Ciliberto2016}, with probability $1-\delta$: $(B) \leq \frac{4\|h^*_{\psi}\|_{\bmH}Q\kappa\log\frac{2}{\delta}}{\lambda \sqrt{n}}$.

[\textbf{Bound (C)}]
We have:
\begin{align*}
    (C) &= \|Z^*S ((C + \lambda I)^{-1} - C^{\dagger})\|_{\infty}\\
    &= \lambda \|Z^*S(C + \lambda I)^{-1}\|_{\infty}\\
    &\leq \lambda \|h^*_{\psi}\|_{\bmH}
\end{align*}
We conclude now by union bound, with probability at least $1-2\delta$:
\[\|H_n - H\|_{\infty} \leq \frac{4Q\kappa\log\frac{2}{\delta}}{\lambda \sqrt{n}} + \frac{4\|h^*_{\psi}\|_{\bmH}Q\kappa\log\frac{2}{\delta}}{\lambda \sqrt{n}} + \lambda \|h^*_{\psi}\|_{\bmH}\]
\end{proof}

\begin{lemma}\label{lemma:HnCnHn} If $B_n = (c H C H^* + tI)^{-\frac{1}{2}}c(H_nC_nH_n^* - HCH^*)(c H C H^* + tI)^{-\frac{1}{2}}$, and $t \geq 5 \max\left(4c^2v_n, 2c v_n, 4cRw_n u_n, 2c w_n^2u_n, 2R^2 u_n\right)$, with probability $1-2\delta $ it is

\[\|B_n\|_{\infty} \leq \frac{1}{2} \]

with $u_n = \frac{4\kappa^2 \log\frac{2}{\delta}}{\sqrt{n}}, v_n = \E_x(\|\hat h_{\psi}(x) - h^*_\psi\|^2)$, $w_n=(\frac{Q}{\lambda \kappa} u_n (1+R) + \lambda R), R= \|h^*_{\psi}\|_{\bmH}$
\end{lemma}

\begin{proof}

Here, we do the following decomposition in 7 terms in order to only have a sum of product of empirical estimators appearing only in term of difference with their ideal target. Then we will bound each associated term in $\|B_n\|_{\infty}= \|(c H C H^* + tI)^{-\frac{1}{2}}c(H_nC_nH_n^* - HCH^*)(c H C H^* + tI)^{-\frac{1}{2}}\|_{\infty}$.
\begin{align*}
    H_nC_nH_n^* - HC_nH^* &= (H_n-H)CH^* \quad (i)\\
    &\quad + H C(H_n-H)^* \quad (ii)\\
    &\quad + (H_n-H)C(H_n-H)^* \quad (iii)\\
    &\quad + (H_n-H)(C_n-C)H^* \quad (iv)\\
    &\quad + H(C_n-C)(H_n-H)^* \quad (v)\\
    &\quad + (H_n-H)(C_n - C)(H_n-H)^* \quad (vi)\\  
    &\quad + H(C_n - C)H^* \quad (vii)  
\end{align*}

[\textbf{Bound} $(i)$ and $(ii)$]
\begin{align*}
    \|(c H C H^* + tI)^{-\frac{1}{2}}cH C(H_n-H)^*(c H C H^* + tI)^{-\frac{1}{2}}\|_{\infty} &\leq c \|(c H C H^* + tI)^{-\frac{1}{2}}HS^*\|_{\infty} \\ 
    &\quad \times \|(c H C H^* + tI)^{-\frac{1}{2}}(H_n-H)S^*\|_{\infty}\\
\end{align*}
But:
\begin{align*}
    \|(c H C H^* + tI)^{-\frac{1}{2}}HS^*\|_{\infty} &= \|(c H C H^* + tI)^{-\frac{1}{2}}HS^*SH^*(c H C H^* + tI)^{-\frac{1}{2}}\|_{\infty}^2\\
    &= \|(c H C H^* + tI)^{-\frac{1}{2}}HCH^*(c H C H^* + tI)^{-\frac{1}{2}}\|_{\infty}^2\\
    &\leq 1
\end{align*}
And:
\begin{align*}
    \|(c H C H^* + tI)^{-\frac{1}{2}}(H_n-H)S^*\|_{\infty} &\leq \frac{1}{\sqrt{t}}\|(H_n-H)S^*\|_{\infty}\\
    &= \frac{1}{\sqrt{t}} \sqrt{\E_x(\|\hat h(x) - h^*(x)\|^2)}
\end{align*}

[\textbf{Bound} $(iii)$]
\begin{align*}
    \|(c H C H^* + tI)^{-\frac{1}{2}}(H_n-H)C(H_n-H)^*(c H C H^* + tI)^{-\frac{1}{2}}\|_{\infty} &\leq \frac{c}{t}\E_x(\|\hat h(x) - h^*(x)\|^2)
\end{align*}

[\textbf{Bound} $(iv)$ and $(v)$]
\begin{align*}
    \|(c H C H^* + tI)^{-\frac{1}{2}}(H_n-H)(C_n-C)H^*(c H C H^* + tI)^{-\frac{1}{2}}\|_{\infty} &\leq \frac{c}{t}\|H_n-H\|_{\infty}\|C_n-C\|_{\infty} \|H\|_{\infty}\\
    &\leq \frac{c}{t} R \|H_n-H\|_{\infty}\|C_n-C\|_{\infty},
\end{align*} noting $R =  \|h^*_{\psi}\|_{\bmH}$.

[\textbf{Bound} $(vi)$]
\begin{align*}
    \|(c H C H^* + tI)^{-\frac{1}{2}}(H_n-H)(C_n - C)(H_n-H)^*(c H C H^* + tI)^{-\frac{1}{2}}\|_{\infty} &\leq \frac{c}{t} \|H_n-H\|_{\infty}^2 \|C_n-C\|_{\infty}
\end{align*}

[\textbf{Bound} $(vii)$]
\begin{align*}
    \|(c H C H^* + tI)^{-\frac{1}{2}}H(C_n - C)H^*(c H C H^* + tI)^{-\frac{1}{2}}\|_{\infty} &\leq \frac{c}{t} R^2 \|C_n-C\|_{\infty}
\end{align*}

[\textbf{Conclusion}]

Hence, noting $u_n=\frac{4\kappa^2 \log\frac{2}{\delta}}{\sqrt{n}}, v_n = \E_x(\|\hat h(x) - h^*(x)\|^2), w_n = \frac{4Q\kappa\log\frac{2}{\delta}}{\lambda \sqrt{n}}(1+\|h^*_{\psi}\|_{\bmH}) + \lambda \|h^*_{\psi}\|_{\bmH}$, if $t \geq 5 \max\left(4c^2v_n, 2c v_n, 4cRw_n u_n, 2c w_n^2u_n, 2R^2 u_n\right)$, by union bound with probability $1-2\delta$

\[\|B_n\|_{\infty} \leq \frac{1}{2}\]

As, with probability $1-\delta$: $\|(C_n-C)\|_{\infty} \leq u_n$ (cf. \cite{Ciliberto2016}, proof of lemma 18.), and also with probability $1-\delta$ (cf. lemma \ref{lemma:Hn}): $\|H_n - H\|_{\infty} \leq w_n$.
\end{proof}

\subsection{Theorem}\label{sec: sup_thm}


In this section we prove the Theorem \ref{thm:excess-risk}.

\thmer*


\begin{proof} First, we bound $\E_x(\|\tilde g \circ \hat h_g(x) - h^*_{\psi}(x) \|^2)$ by decomposing it in two parts. We have, defining $\hat P = \hat G^* \hat G$,
\begin{align}
    \E_x(\|\tilde g \circ \hat h_g(x) - h^*_{\psi}(x) \|^2) &= \E_x(\|\hat P \hat h_{\psi}(x) - h^*_{\psi}(x)\|^2)\\
    &= \underbrace{\E_x(\|\hat P (\hat h_{\psi}(x) - h^*_{\psi}(x))\|^2)}_{(1)} + \underbrace{\E_x(\|\hat P h^*_{\psi}(x) - h^*_{\psi}(x)\|^2)}_{(2)}
\end{align}
We now bound the two terms of equations (19).

[\textbf{Bound} $(1)=\E_x(\|\hat P (\hat h_{\psi}(x) - h^*_{\psi}(x))\|^2)$]
We upper bound this term using Lemma \ref{krr}.

[\textbf{Bound} $(2)=\E_x(\|\hat P h^*_{\psi}(x) - h^*_{\psi}(x)\|^2)$]
We upper bound this term using two dedicated lemmas, first Lemma~\ref{jensen}, then Lemma~\ref{set}, with $t = \max(\sigma_k(N), t_1, t_2)$. If $p^r \leq \min(\omega/t_1, \omega/t_2^{n})$, then $\sigma_k(N) \geq \max(t_1, t_2)$, so $t = \sigma_k(N)$, and with probability $1-3\delta$:
\begin{align*}
    \sqrt{\E_x(\|\hat P h^*_{\psi}(x) - h^*_{\psi}(x)\|^2)} &\leq \sqrt{\langle I - \hat G^*\hat G,\, M_c \rangle_{\bmHy \otimes \bmHy}}\\
    &\leq 3 \Omega' t^{-1/2r} \sqrt{2 \sigma_k(M_c)}\\
    &= 3\sqrt{2} \Omega' \sigma_k(M_c)^{1/2(1-1/r)}\\
    &\leq \sqrt{\Omega q_r} \times p^{(1-r)/2}
\end{align*}

with: $q_r = 36 \times \Gamma(1 - 1/r)\Gamma(1 + 1/r)\Gamma(1/r)$

We conclude the desired bound by union bound.
\end{proof}

\coro*
\begin{proof}

We have: $u_n = \mathcal{O}(n^{-1/2})$, and using $\lambda = 1 / \sqrt{n}$, we have $v_n = \E_x(\|\hat h(x) - h^*(x)\|^2) = \mathcal{O}(n^{-1/2})$, and $w_n = \mathcal{O}(1)$, so we can use a number of components $p = \mathcal{O}(n^{1/2r})$ and having the condition on $p$ of the theorem verified: $p^r \leq  \min\left\{\frac{\omega m}{9 (1-c) \log(m/\delta)},\frac{\omega}{8 c t_n}\right\}$, if $m \geq n$.

Now, injecting $p = \mathcal{O}(n^{-1/2})$ in the reconstruction error term we get the desired $O(n^{-(r-1)/4r})$.
\end{proof}

It's interesting to note that when $ r $ increases $O(n^{-(r-1)/4r})$ becomes really close to the typical rate of $O(n^{-1/4})$ but with a number of components $p \ll n$.

\subsection{Time and Space Complexity Analysis of OEL}
\textbf{Train.} Algorithm 1 is used to obtain, thanks to the training data, $\hat \beta \in \R^{(n+m) \times p}$ and $W = (K_x + n\lambda I)^{-1} \in \R^{n \times n}$ defining $\hat G$ and $\hat h_{\psi}$ respectively. The complexities of algorithm 1 are given by summing the complexities of computing $\beta$ and $W$. We give in Tables \ref{time_complexity} and \ref{space_complexity} the time and space complexities for these computations. Both can be solved using standard approximation methods, and we also give the complexity of algorithm 1 when using Nystr\"{o}m KRR approximation of rank $ q $ and randomized SVD approximation of rank $p$.

\textbf{Test.} During the test phase, we use $\beta$ and $W$ computed during the training phase in order to compute $\hat f(x) \in \bmY$ for any test point $x$ (see below). The decoding part is computationally very expensive, in general, as it requires an exhaustive search in the candidate sets which can be very large, with costly distance computations. Noting $\alpha(x) = (K_x + n \lambda I_n)^{-1}k_x \in \R^n$, $k_x = (k_x(x, x_1), ..., k_x(x, x_n))$, $k_y = (k_y(y, y_1), ..., k_y(y, y_n))$, $K_y \in \R^{n \times n}$ the training output gram matrix, $\beta \in \R^{(n+m) \times p}$ the coefficient matrix obtained with algorithm \ref{algo:linear}, the decoding computation for OEL is:

\begin{align*}
    \hat f (x) &= \argmin\limits_{y \in \bmY} \|\hat G^* \hat G \hat h_{\psi}(x) - \psi(y)\|^2_{\bmH_y}\\
    &= \argmin\limits_{y \in \bmY} \|\hat G^* \hat G \hat h_{\psi}(x)\|^2_{\bmH_y} +  \|\psi(y)\|^2_{\bmH_y} - 2 \langle \hat G^*\hat G \hat h_{\psi}(x)\, \psi(y)\rangle_{\bmH_y}\\
    &= \argmin\limits_{y \in \bmY} \|\psi(y)\|^2_{\bmH_y} - 2 \langle \hat G \hat h_{\psi}(x)\,\hat G \psi(y) \rangle_{\R^p}\\
    &= \argmin\limits_{y \in \bmY} \|\psi(y)\|^2_{\bmH_y} - 2 \alpha(x)^T K_y \beta \beta^T k_y\\
\end{align*}

Computing the $\|\psi(y)\|^2_{\bmH_y}$ for all candidates is $|\bmY|$ kernel evaluations. Computing the projected training and candidates points $K_y \beta \in \R^{n \times p}$ and $\beta^T k_y \in \R^{|\bmY| \times p}$ costs $n^2p + n|\bmY|p = \mathcal{O}(np|\bmY|)$ (as $n< |\bmY|$). The previous operations can be done only one time for the whole test set. Then, for one test point $x$, computing $\alpha(x)^T \times  K_y \beta  \times \beta^T k_y \in \R$ for all candidates costs $n + p|\bmY| = \mathcal{O}(p|\bmY|)$ (as $n< |\bmY|$). Finally, the decoding complexity for all test points is $\mathcal{O}(np|\bmY| + n_{te} p|\bmY|) = \mathcal{O}(n_{te} p|\bmY|)$ (if $p< n_{te}$).

When decoding with IOKR, instead of $\alpha(x)^T K_y \beta \beta^T k_y$, we need to compute $\alpha(x)^T k_y \in \R$ for all candidates which costs $n|\bmY|$. Finally, the decoding complexity for all test points is $n_{te}n|\bmY|$.

\begin{table}[!ht]
  \centering
  \begin{tabular}{lll}
    \toprule
    Algorithm & KRR & SVD \\
    \midrule
     Standard & $\bmO(n^3)$ &  $\bmO((n + m)^3)$ \\
     Approximated & $\bmO(n q^2 + q^3)$ & $\bmO((n + m)^2 p + (n + m)p^2)$ \\
    \bottomrule
  \end{tabular}
\caption{Training time complexity of algorithm 1}
\label{time_complexity}
\end{table}

\begin{table}[!ht]
  \centering
  \begin{tabular}{lll}
    \toprule
    Algorithm & KRR & SVD \\
    \midrule
    Standard & $\bmO(n^2)$ & $\bmO((n + m)^2)$ \\
    Approximated & $\bmO(q^2 + n q)$ & $\bmO((n + m) p)$ \\
    \bottomrule
  \end{tabular}
\caption{Training space complexity of algorithm 1}
\label{space_complexity}
\end{table}

\begin{table}[!ht]
  \centering
  \begin{tabular}{ll}
    \toprule
     & Decoding \\
    \midrule
     IOKR & $\bmO(n_{te} n|\bmY|)$\\
     OEL & $\bmO(n_{te} p|\bmY| + n|\bmY|p) \approx~ \bmO(n_{te} p|\bmY|)$\\
    \bottomrule
  \end{tabular}
\caption{Decoding time complexity of algorithm 1}
\label{time_complexity_decoding}
\end{table}

\textbf{Experimental computational time evaluation.}\label{sec: sup_time}

In the following Table \ref{tab:experimental_computation_time} we give the fitting and decoding time of IOKR and OEL for the experiments of Table \ref{usps_table1} and \ref{tab:multilabel_comparison}. In the Table \ref{usps_table1} the setup is $n = 1000$ training couples $(x, y)$, $m=6000$ alone outputs $y$, $n_{te}=2007$ test data, $7000$ candidates for decoding (the training outputs), and $p=98$ for OEL. In the Table \ref{tab:multilabel_comparison} the setup for Bibtex and Bookmarks are, respectively, $n = 4880/60000$ training couples $(x, y)$, $m=0/0$ alone output $y$, $n_{te}=2007/27856$ test data, $4880/13779$ candidates for decoding (the training outputs), and $p=130/200$ for OEL.

\begin{table}[!ht]
  \centering
  \begin{tabular}{llll}
    \toprule
     & IOKR & OEL$_0$ & OEL \\
    \midrule
    Bibtex & 2s/13s & 15s/4s &  Na \\
    Bookmarks & 465s/3371s & 617s/214s & Na \\
    USPS & 0.1s/9s & 0.4s/1s & 37s/38s \\
    \bottomrule
  \end{tabular}
\caption{Fitting/Decoding computation time of IOKR compared to OEL (in seconds)}
\label{tab:experimental_computation_time}
\end{table}

In comparison with mere IOKR, OEL$_0$ and OEL necessitate an extra training: the linear subspace estimation. However, OEL alleviates the time complexity of the decoding. In the Table \ref{tab:experimental_computation_time}, we see that when $n, n_{te}$, and $|\bmY|$ (the number of candidates) are big, and $p \ll n$, OEL leads to a significant improvement (see Bookmarks). Alleviating the testing time even by increasing training time is an interesting property. OEL training (with $m$ extra data) in comparison with OEL$_0$ leads to a greater computation time but better statistical performance (cf. USPS results Table \ref{usps_table1}, and Table \ref{tab:multilabel_small}).

\subsection{Additional Experimental Results and Details}\label{sec: sup_exp}

\subsubsection{Image Reconstruction}

\paragraph{Experimental setting.}  As in \cite{weston2003kernel} we used as target loss an RBF loss $ \|\psi(y) - \psi(y')\|^2 $ induced by a Gaussian kernel $ k $ and visually chose the kernel's width $ \sigma_{output}^2 = 10 $ looking at reconstructed images of IOKR without embedding learning. We constituted a supervised training set with the first 1000 train digits, and an unsupervised training set with the 6000 last bottom half train digits. We used a Gaussian input kernel of width $\sigma_{input}$. For the pre-image step, we used the same candidate set for all methods constituted with all the 7000 training bottom half digits. We selected the hyper-parameters $ \gamma_{input}, p, \lambda $ using logarithmic grids and the supervised/unsupervised balance parameter $ \gamma $ using linear grid, via 5 repeated random sub-sampling validation (80\%/20\%) selecting the best mean validation MSE, then we trained a model on the entire training set, and we tested on the test set.



\textbf{Link to downloadable dataset } \url{https://web.stanford.edu/~hastie/StatLearnSparsity_files/DATA/zipcode.html}

\paragraph{SPEN USPS experiments' details.} We used an implementation of SPEN in python with PyTorch by Philippe Beardsell and Chih-Chao Hsu (cf. https://github.com/philqc/deep-value-networks-pytorch). Small changes have been made. 
SPEN was trained using standard architecture from \cite{belanger2016}, that is a simple 2-hidden layers neural network for the feature network with equal layer size $n_h=110$, and a single-hidden layer neural network for the structure learning network with size $n_s=50$. The size of the two hidden layers $n_h \in [10,30,50,70,90,110,130]$ was selected during the pre-training of the feature network using 5 repeated random sub-sampling validation (80\%/20\%) selecting the best mean validation MSE (cf. figure \ref{fig:convergence} for convergence of this phase). $n_s \in [5, 10, 20, 50, 70]$ was selected during the training phase of the SPEN network (training of the structure learning network plus the last layer of the feature network) doing approximate loss-augmented inference (cf. figure \ref{fig:convergence}  for inferences' convergences), and minimizing the SSVM loss, using 5 repeated random sub-sampling validation (80\%/20\%) selecting the best mean validation MSE (cf. figure \ref{fig:convergence} for convergence of this phase).

\begin{figure}[!ht]
    \centering
    \begin{minipage}[l]{0.33\textwidth}
        \centering
        \includegraphics[height=0.17\textheight]{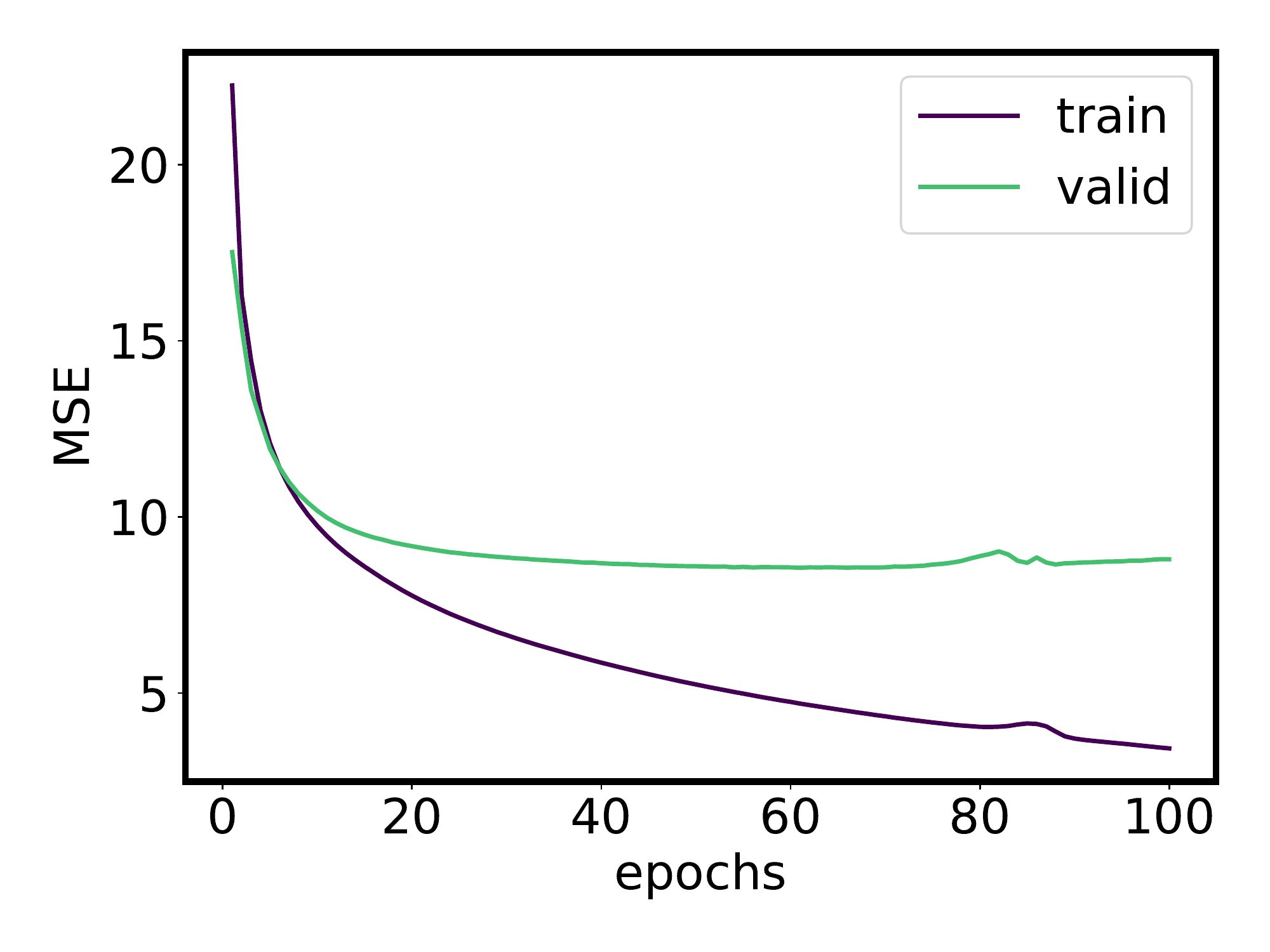}
    \end{minipage}%
    \begin{minipage}[l]{0.33\textwidth}
        \centering
        \includegraphics[height=0.17\textheight]{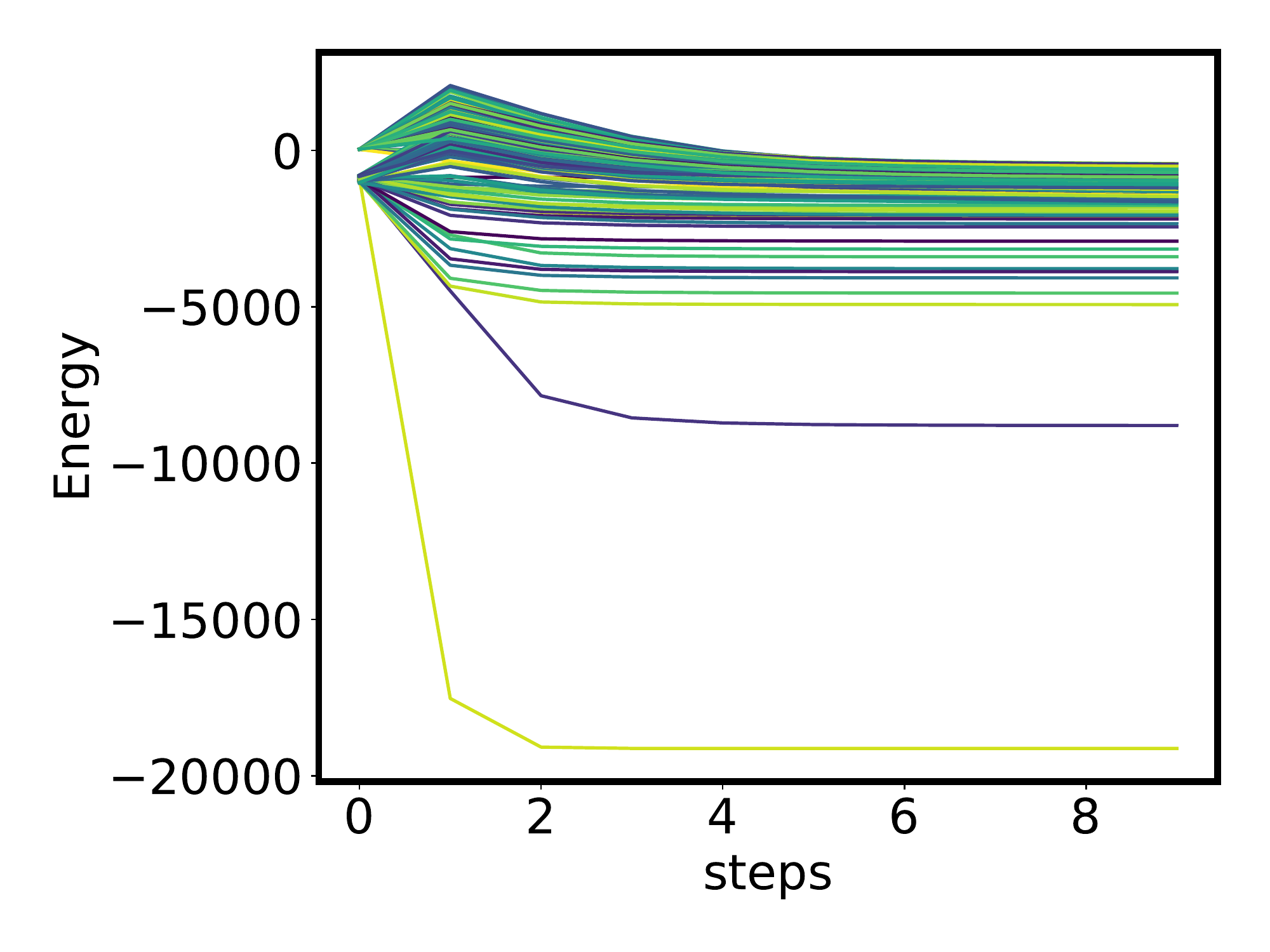}
    \end{minipage}%
    \begin{minipage}[l]{0.33\textwidth}
        \centering
        \includegraphics[height=0.17\textheight]{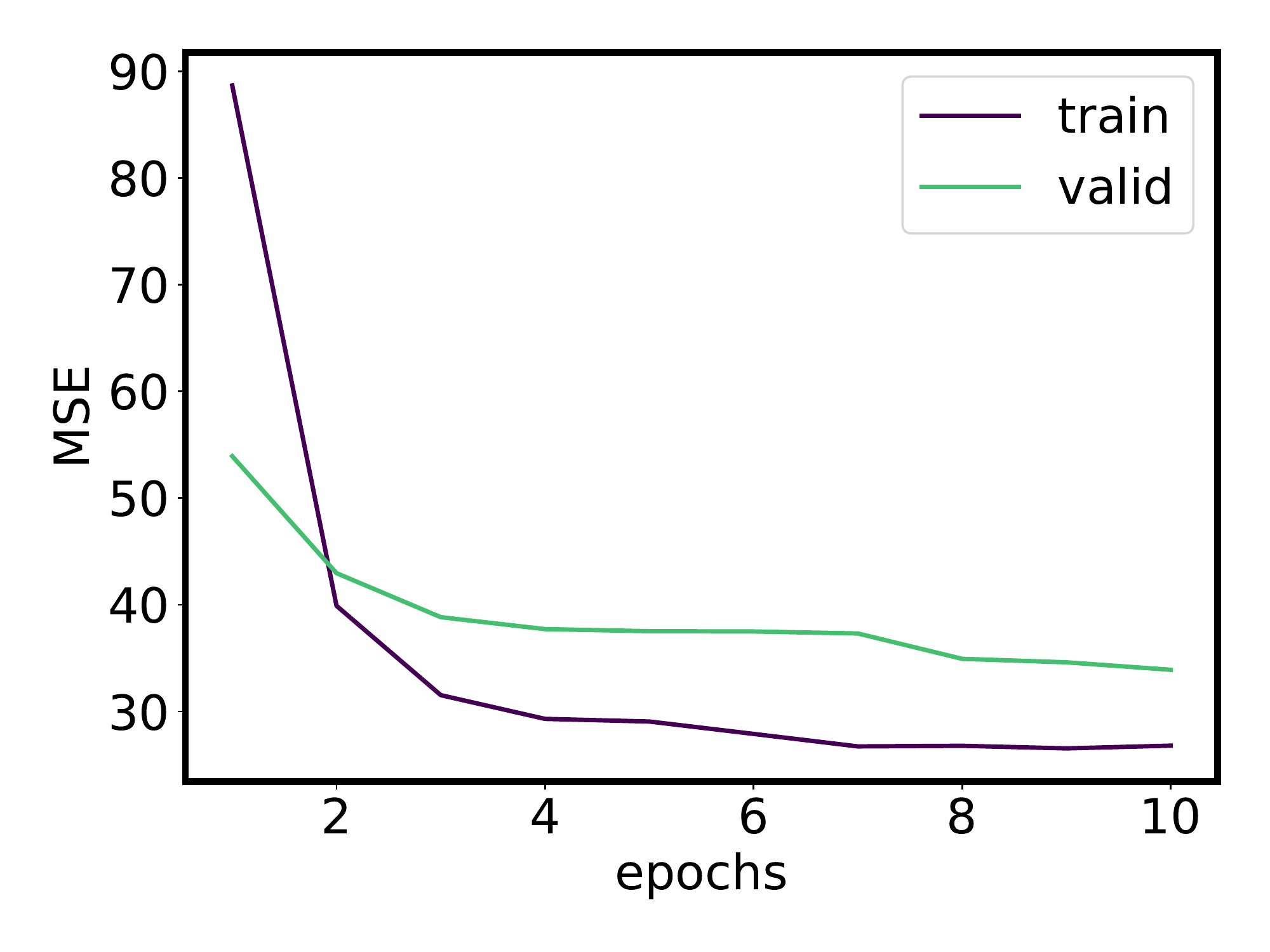}
    \end{minipage}
    \caption{Left: Convergence of train/validation MSE when pre-training the feature network / Center: approximate loss-augmented inferences' convergences/ Right: Convergence of train/validation SSVM loss when training the SPEN network}
    \label{fig:convergence}
\end{figure}

\subsubsection{Multi-label classification}

\paragraph{Problem and dataset}
Bibtex and Bookmarks \citep{katakis2008} are tag recommendation problems, in which the objective is to propose a relevant set of tags (e.g. url, description, journal volume) to users when they add a new Bookmark (webpage) or Bibtex entry to the social bookmarking system Bibsonomy. Corel5k is an image dataset and the goal of this application is to annotate these images with keywords. Information on these datasets is given in Table \ref{multilabel_dataset}.

\begin{table}[!ht]
  \centering
  \begin{tabular}{llllll}
    \toprule
    Dataset & $n$ & $n_{te}$ & $n_{features}$ & $n_{labels}$ & $\bar l$\\
    \midrule
     Bibtex & 4880 & 2515 & 1836 & 159 & 2.40\\
     Bookmarks & 60000 & 27856 & 2150 & 208 & 2.03\\
     Corel5k & 4500 & 499 & 37152 & 260 & 3.52\\
    \bottomrule
  \end{tabular}
\caption{Multi-label datasets description. $\bar l$ denotes the averaged number of labels per point.}
\label{multilabel_dataset}
\end{table}

\paragraph{Experimental setting}
For all multi-label experiments we used a Gaussian output kernel with widths $\sigma^2_{output} = \frac{1}{\bar l}$, where $\bar l$ is the averaged number of labels per point.  As candidate sets we used all the training output data. We measured the quality of predictions using example-based F1 score. We selected the hyper-parameters $\lambda$ and $p$ in logarithmic grids. 

\textbf{Link to downloadable dataset } \url{http://mulan.sourceforge.net/datasets-mlc.html}

\paragraph{About the selected output embeddings' dimensions p}

We selected the output embeddings' dimensions $p$ with integer logarithmic scales, ensuring that the selected dimensions were always smaller than the maximal one of the grids. In multilabel experiments (cf. Table \ref{tab:multilabel_small} and \ref{tab:multilabel_comparison}), we obtained the following dimensions (cf. Table \ref{selected_p}).

\begin{table}[!ht]
  \centering
  \begin{tabular}{lll}
    \toprule
    Dataset & Table \ref{tab:multilabel_small} & Table \ref{tab:multilabel_comparison} \\
    \midrule
     Bibtex & 80/80 & 130\\
     Bookmarks & 30/40 & 200\\
     Corel5k & 24/162 & Na\\
    \bottomrule
  \end{tabular}
\caption{Selected output embeddings' dimensions $p$ with OEL$_0$/OEL in Table \ref{tab:multilabel_small} and OEL in Table \ref{tab:multilabel_comparison}}
\label{selected_p}
\end{table}

In Table \ref{tab:multilabel_small} recall that we used a reduced number of training couples, which allows to have alone training outputs. If we interpret $p$ as a regularisation parameter, we see that when $n$ increases (from Table \ref{tab:multilabel_small} to Table \ref{tab:multilabel_comparison}) or  $m$ increases (from OEL$_0$ to OEL in Table \ref{tab:multilabel_small}), then there is less need for regularisation hence $p$ is bigger.

\subsubsection{Metabolite identification}

\paragraph{Problem and dataset}

An important problem in metabolomics is to identify the small molecules, called metabolites, that are present in a biological sample. Mass spectrometry is a widespread method to extract distinctive features from a biological sample in the form of a tandem mass (MS/MS) spectrum.  
In output the molecular structures of the metabolites are represented by fingerprints, that are binary vectors of length $d = 7593$. Each value of the fingerprint indicates the presence or absence of a certain molecular property.
 Labeled data are expensive to obtain, but a very large unsupervised dataset (several millions, 6455532 in our case) is available in output. For each input the molecular formula of the output is assumed to be known, and we consider all the molecular structures having the same molecular formula as the corresponding candidate set. 
The median size of the candidate sets is 292, and the biggest candidate set is of size 36918.

\paragraph{Experimental setting}
The dataset contains 6974 supervised data $(x_i , y_i)$ and several millions of unlabeled data are available in output. In input we use a probability product kernel on the tandem mass spectra. 
As output kernel we used a a Gaussian kernel (with parameter $\sigma^2=1$) in which the distances are taken between feature vectors associated with a Tanimoto kernel. When no additional unsupervised data are used,  we selected the hyper-parameters $\lambda, p $ in logarithmic grids using nested cross-validation with 5 outer folds and 4 inner folds. In the case of OEL with $10^5$ additional unsupervised data, we fixed $p=2000$, and selected $\lambda, \gamma$ with 5 outer folds and 4 inner folds but only using $5 \times 10^3$ additional data.

\paragraph{SPEN metabolite identification experiments' details}
For this problem, we first used kpca in order to compute finite input representations of the mass spectra (with not too big dimension $p=2000$). For the purpose of checking the quality of these finite inputs, we used them with IOKR (with linear output kernel), just computing the top-k accuracies on test inputs with less than 300 candidates for faster computations. Using these finite inputs' representations leads effectively to comparable results, and even better top-1 accuracy (cf. Table \ref{kpca_input}). We choose $p=2000$. Notice that bigger $ p $ would results in more difficult SPEN optimization.

\begin{table}[ht]
  \centering
  \begin{tabular}{llll}
    \toprule
    Input       & MSE & Hamming & Top-k accuracies\\
    & &  & $k=10$ | $k=5$| $k=1$\\
    \midrule
    PPK kernel $p = + \infty$ &  $206.09$ & 111.95 &$ 59.83\%\, | 45.64\%\, |16.92\%$ \\
    KPCA $p = n = 5579$  &  $206.09$ & 111.95 &$ 59.83\%\, |45.64\%\, |16.92\%$ \\
    KPCA $p = 2000$  &  $ 212.6$ & 114.40 &$57.26\%\, |42.74\%\, |17.26\%$ \\
    \bottomrule
  \end{tabular}
    \caption{Test mean losses and standard errors for the metabolite identification problem with IOKR (linear output kernel) and different inputs' representations.}
    \label{kpca_input}
\end{table}

We used the same architecture than for USPS. Similarly to McCallum, we did not tune the sizes of the hidden layers for the feature network ($n_h=1500$), but set them based on intuition and the size of the data, the number of training examples, etc. Then, we train the SPEN network with $n_s \in [50, 100, 500, 1000]$ which gave comparable results and we kept the best one in terms of top-k accuracies, that is $n_s=500$.

\subsection{Label Ranking}\label{ranking}

The goal of label ranking is to learn to rank $ K $ items indexed by {1, . . . , K}. A ranking can be seen as a permutation, i.e a bijection $\sigma : \llbracket 1, K \rrbracket\rightarrow  \llbracket 1, K \rrbracket $ mapping each item to its rank. $i$ is preferred over $j$ according to $ \sigma $ if and only if i is ranked lower than j: $\sigma(i) < \sigma(j)$. The set of all permutations over $ K $ items is the symmetric group which we denote by $S_k$, and can be seen as a structured objects set (cf. \cite{korba2018structured}).

\cite{korba2018structured} have shown that IOKR is a competitive method with state of the art label ranking methods. We evaluate the performance of OEL on benchmark label ranking datasets. Following \cite{korba2018structured} we embedded the permutation using Kemeny embedding. We trained regressors using Kernel ridge regression (Gaussian kernel). We adopt the same setting as \cite{korba2018structured} and report the results of our predictors in terms of mean Kendall’s $\tau$ from five repetitions of a ten-fold cross-validation (c.v.). We also report the standard deviation of the resulting scores. The parameters of our regressors and output embeddings learning algorithms were tuned in a five folds inner c.v. for each training set. 

The results are given in Table~\ref{table_ranking}. Learning a linear output embedding from Kemeny embedding shows a small improvement in term of Kendall’s $\tau$ compared to IOKR. This improvement is observed on most of the datasets.


\begin{table}[!ht]
  \centering
  \begin{tabular}{llllll}
    \toprule
    Method & cold & diau & dtt & heat & sushi\\
    \midrule
    IOKR & $ 0.097 \pm 0.033 $ & $ 0.228 \pm 0.023 $ & $ \mathbf{0.135 \pm 0.036} $ &  $ \mathbf{0.058 \pm 0.020} $ & $ 0.321 \pm 0.021 $ \\
    OEL$_0$ & $ \mathbf{0.105 \pm 0.029} $ & $ \mathbf{0.231 \pm 0.026} $ & $ 0.132 \pm 0.036 $ &  $ 0.056 \pm 0.021 $ & $ \mathbf{0.325 \pm 0.027} $ \\
    \bottomrule
  \end{tabular}
\caption{Mean Kendall's $\tau$ coefficient (higher is better) obtained with IOKR and OEL methods on several label ranking datasets.}
\label{table_ranking}
\end{table}

\end{document}